\documentclass{article}
\usepackage{kr}
\usepackage{times}
\usepackage{soul}
\usepackage{url}
\usepackage[hidelinks]{hyperref}
\usepackage[utf8]{inputenc}
\usepackage[small]{caption}
\usepackage{graphicx}
\usepackage{amsmath}
\usepackage{amsthm}
\usepackage{booktabs}
\usepackage{algorithm}
\usepackage{algorithmic}
\usepackage{subcaption}
\usepackage{color}

\newtheorem{theorem}{Theorem}

\newtheorem{definition}{Definition}

\newcommand\shrink[1]{}

\def\n(#1){\bar{#1}}

\def\X{{\bf X}}
\def\x{{\bf x}}

\def\y{{\bf y}}

\def\eql(#1,#2){{#1\!\!=\!#2}}

\def\eql(#1,#2){{#1\!=\!#2}}
\newcommand\name[1]{\ensuremath{\mathsf{#1}}}
\def\true{\name{true}}
\def\false{\name{false}}

\def\clap#1{\hbox to 0pt{\hss#1\hss}}

\newcommand\Amain[1]{
\textbf{main:}
#1
\vspace{1mm}
}

\newcommand\Ainput[1]{
\vspace{1mm}
\textbf{input:} #1
}

\newcommand\Aoutput[1]{
\vspace{1mm}
\textbf{output:} #1
\vspace{1mm}
}

\def\robust{{\mathsf{r}}}
\def\modelrobust{{\mathsf{mr}}}
\def\maxrobust{{\mathsf{maxr}}}
\def\modelcount{{\mathsf{model\_count}}}
\def\gek(#1){h_{{#1}}}

\title{On Tractable Representations of Binary Neural Networks}

\author{%
Weijia Shi\and
Andy Shih\and
Adnan Darwiche\and
Arthur Choi \\
\affiliations
Computer Science Department, University of California, Los Angeles, California, USA\\
\emails
{\{swj0419,andyshih,darwiche,aychoi\}@cs.ucla.edu}
}

\begin{document}

\maketitle

\begin{abstract}
We consider the compilation of a binary neural network's decision function into tractable representations such as Ordered Binary Decision Diagrams (OBDDs) and Sentential Decision Diagrams (SDDs).  Obtaining this function as an OBDD/SDD facilitates the explanation and formal verification of a neural network's behavior. First, we consider the task of verifying the robustness of a neural network, and show how we can compute the expected robustness of a neural network, given an OBDD/SDD representation of it.  Next, we consider a more efficient approach for compiling neural networks, based on a pseudo-polynomial time algorithm for compiling a neuron. We then provide a case study in a handwritten digits dataset, highlighting how two neural networks trained from the same dataset can have very high accuracies, yet have very different levels of robustness.  Finally, in experiments, we show that it is feasible to obtain compact representations of neural networks as SDDs.
\end{abstract}

\section{Introduction} \label{sec:intro}

Recent progress in artificial intelligence and the increased deployment of AI systems have highlighted the need for \emph{explaining} the decisions made by such systems; see, e.g., \cite{BaehrensSHKHM10,lime:kdd16,anchors:aaai18,Lipton18,ShihCD18,IgnatievNM19b,DarwicheHirth20a}.\footnote{It is now recognized that opacity, or lack of explainability is ``one of the biggest obstacles to widespread adoption of artificial intelligence'' (The Wall Street Journal, August 10, 2017).} For example, one may want to explain {\em why} a classifier decided to turn down a loan application, or rejected an applicant for an academic program, or recommended surgery for a patient. Answering such {\em why?} questions is particularly central to assigning blame and responsibility, which lies at the heart of legal systems and is further a requirement in certain contexts.\footnote{Take for example the European Union general data protection regulation, which has a provision relating to explainability, \url{http://www.privacy-regulation.eu/en/r71.htm}. \label{footnote:eu}}  The {\em formal verification} of AI systems has also come into focus recently, particularly when such systems are deployed in safety-critical applications.

We propose a \emph{knowledge compilation} approach for explaining and verifying the behavior of a neural network classifier.  Knowledge compilation is a sub-field of AI that studies in part \emph{tractable} Boolean circuits, and the trade-offs between succinctness and tractability \cite{SelmanK96,CadoliD97,darwicheJAIR02,Darwiche14}.  By enforcing different properties on the structure of a Boolean circuit, one can obtain greater tractability (the ability to perform certain queries and transformations in polytime) at the possible expense of succinctness (the size of the resulting circuits). Our goal is to compile the Boolean function specified by a neural network into a tractable Boolean circuit that facilitates explanation and verification.

We consider neural networks whose inputs are binary (\(0/1\)) and that use step activations. Such a network would have real-valued parameters, but the network itself induces a purely Boolean function.  We seek a \emph{tractable Boolean circuit} that represents this function, which we obtain in two steps.
First, note that neurons with step activations and binary inputs then produce a binary output---each neuron induces its own Boolean function.  Using, e.g., the algorithm of \cite{chanUAI03} we can obtain a tractable circuit for a given neuron's Boolean function.  The neural network then induces a Boolean circuit, although it may not be tractable.  Thus, we \emph{compile} this circuit into a tractable one by enforcing additional properties on the circuit until certain operations become tractable, as done in the field of knowledge compilation.  We then explain the decisions and verify the properties of this circuit, as done in \cite{ShihCD18,ShihCD18b}; cf. \cite{DarwicheHirth20a}.

Our approach follows a recent trend in analyzing machine learning models using symbolic approaches such as satisfiability and satisfiability modulo theory; see, e.g., \cite{KatzBDJK17,Leofante18,NarodytskaKRSW18,ShihCD18,IgnatievNM19a,IgnatievNM19b,SDC19,Audemardetal20}.  While machine learning and statistical methods are key for learning classifiers, it is clear that symbolic and logical approaches, which are independent of any of the models parameters, are key for analyzing and reasoning about them.  Our approach, based on compilation into a tractable Boolean circuit, can go beyond queries based on (for example) satisfiability, as we shall show.

This paper is organized as follows.  In Section~\ref{sec:preliminaries}, we review relevant background material.  In Section~\ref{sec:neurons}, we show how to reduce neural networks to Boolean circuits by compiling each neuron into a
Boolean circuit.  In Section~\ref{sec:kc}, we discuss how to obtain tractable circuits, via knowledge compilation.  In Section~\ref{sec:robustness}, we show how a tractable circuit enables one to reason about the robustness of a neural network.
In Section~\ref{sec:case-study}, we provide a case study with experimental results and finally conclude with a discussion in Section~\ref{sec:conclusion}.

\section{Technical Preliminaries} \label{sec:preliminaries}

\begin{figure}[tb]
\centering
\begin{subfigure}[b]{\linewidth}
\centering
\includegraphics[width=.6\linewidth]{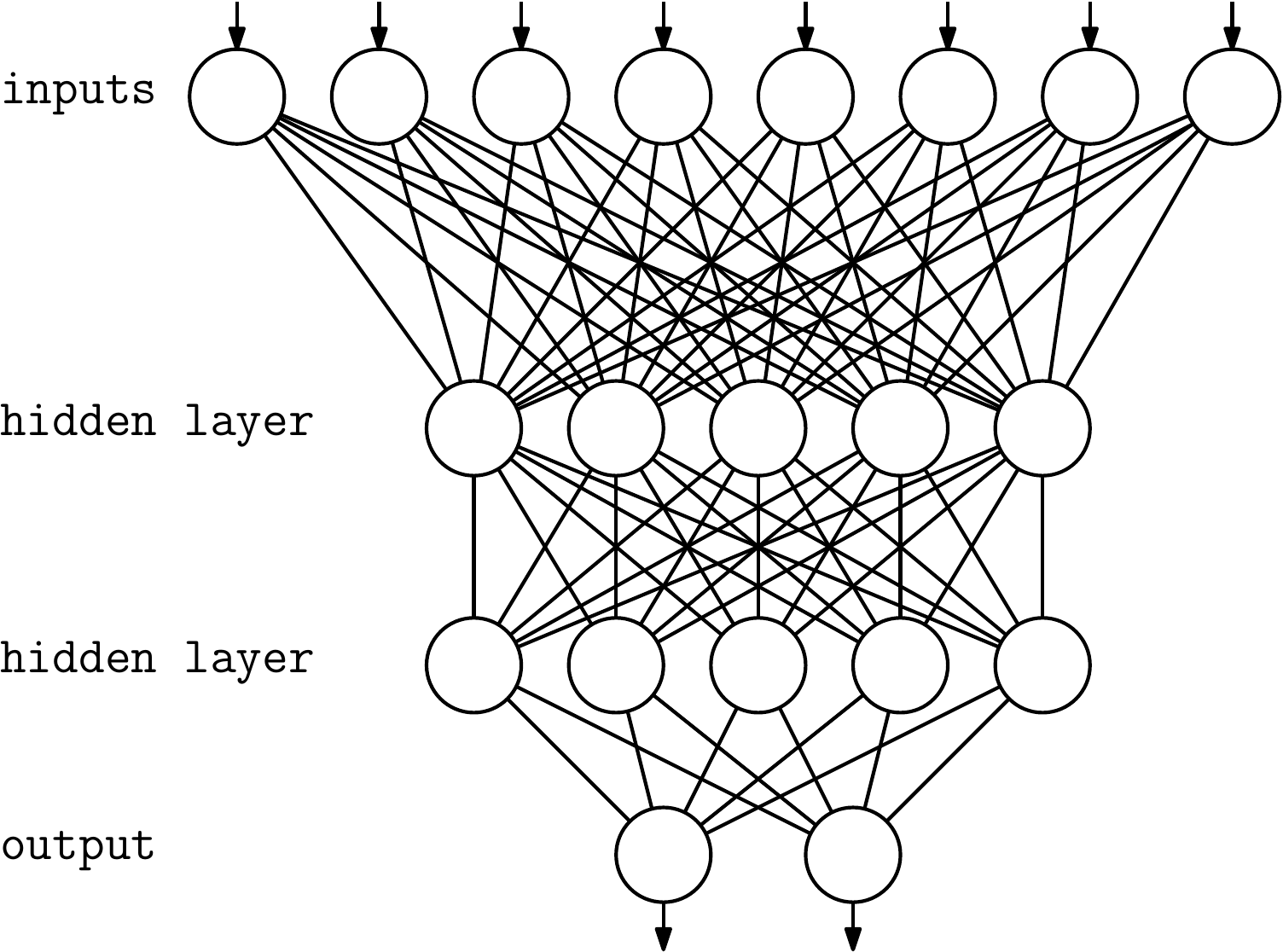}
\caption{A neural network structure}\label{fig:structure}
\end{subfigure} \\ \bigskip
\begin{subfigure}[b]{\linewidth}
\centering
\includegraphics[width=.6\linewidth]{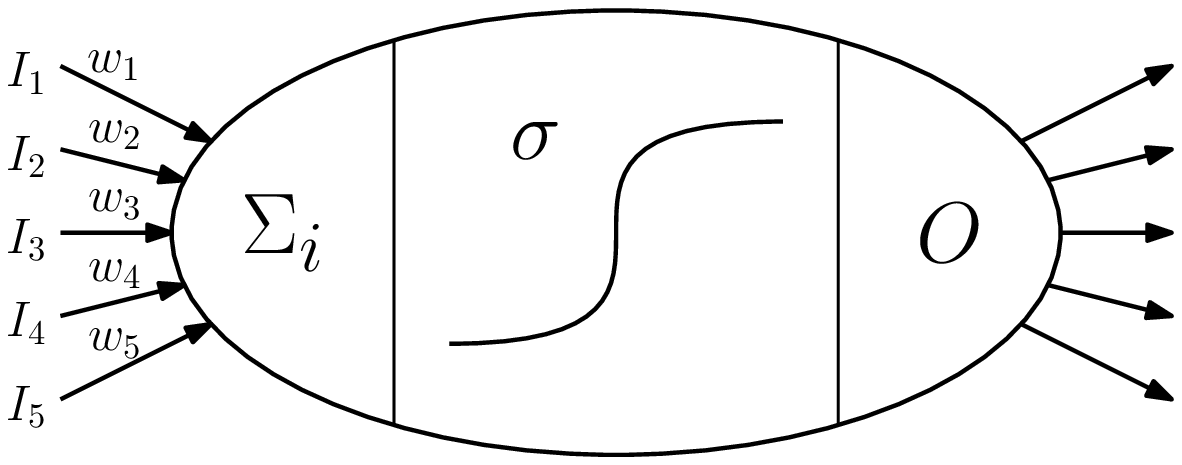}
\caption{A mathematical model of a neuron}\label{fig:neuron}
\end{subfigure} \\ \bigskip
\begin{subfigure}[b]{\linewidth}
\centering
\end{subfigure}
\caption{A neural network and a neuron.
A sigmoid activation \(\sigma(x) = \frac{1}{1+\exp\{-x\}}\) acts as a soft threshold which tends to 0 and 1 as \(x\) goes to \(-\infty\) and \(\infty,\) respectively. A ReLU activation \(\sigma(x) = \max(0,x)\) outputs 0 if \(x < 0\) and outputs \(x\) otherwise.
 \label{fig:neural-network}}
\end{figure}

A feedforward neural network is a directed acyclic graph (DAG); see Figure~\ref{fig:structure}. The roots of the DAG are the neural network inputs, call them \(X_1, \ldots, X_n\). The leaves of the DAG are the neural network outputs, call them \(Y_1, \ldots, Y_m\). Each node in the DAG is called a {\em neuron} and contains an {\em activation function} \(\sigma\); see Figure~\ref{fig:neuron}. Each edge \(I\) in the DAG has a {\em weight} \(w\) attached to it. The weights of a neural network are its {\em parameters,} which are learned from data.

In this paper, we assume that the network inputs \(X_i\) are either \(0\) or \(1\).  We further assume \emph{step} activation functions:
\[
\sigma(x) =
\left\{
\begin{tabular}{rl}
1 & if \(x \ge 0\) \\
0 & otherwise
\end{tabular}
\right.
\]
A neuron with a step activation function has outputs that are also \(0\) or \(1\).  If the network inputs are also \(0\) or \(1\), then this means that the inputs to all neurons are \(0\) or \(1\).  Moreover, the output of the neural network is also \(0\) or \(1\).  Hence, each neuron and the network itself can be viewed as a function mapping binary inputs to a binary output, i.e., a Boolean function.  For each neuron, we shall simply refer to this function as the \emph{neuron's Boolean function}.  When there is a single output \(Y\), we will simply refer to the corresponding function as the \emph{network's Boolean function}.

\section{From Neural Networks to Boolean Circuits} \label{sec:neurons}

\begin{figure}[tb]
\centering
\begin{subfigure}[b]{.35\linewidth}
\centering
\includegraphics[width=.6\linewidth]{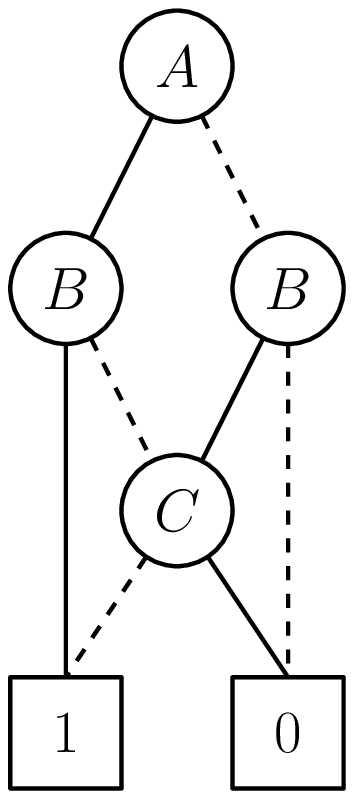}
\caption{An OBDD}\label{fig:neuron-obdd}
\end{subfigure}
\begin{subfigure}[b]{.35\linewidth}
\centering
\includegraphics[width=.8\linewidth]{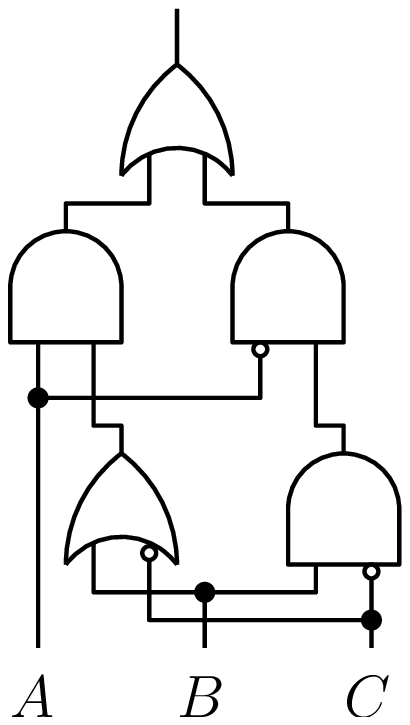}
\caption{A circuit}\label{fig:neuron-circuit}
\end{subfigure}
\caption{An OBDD and circuit representation of a neuron \(\sigma(A + B -C - 1)\) where \(\sigma\) is a step activation function.
 \label{fig:neuron-decision}}
\end{figure}

Consider a neuron with step activation function \(\sigma\), inputs \(I_{i}\), weights \(w_{i}\) and bias \(b\). The output of this neuron is simply
\begin{equation} \label{eq:neuron}
\sigma(\sum_{i} w_{i} \cdot I_{i} + b) =
\left\{
\begin{tabular}{rl}
1 & if \(\sum_{i} w_{i} \cdot I_{i} + b \ge 0\) \\
0 & otherwise
\end{tabular}
\right.
\end{equation}
As an example, consider a neuron with 3 inputs \(A,B\) and \(C\) with weights \(w_1 = 1.15, w_2 = 0.95\) and \(w_3 = -1.05\) and a bias of \(-0.52\). This neuron outputs 1 iff:
\[
1.15 \cdot A + 0.95 \cdot B - 1.05 \cdot C \ge 0.52 
\]
Treating a value of 1 as \true\ and a value of 0 as \false, we can view this neuron as a Boolean function \(f(A,B,C)\) whose output matches that of the neuron, on inputs \(A,B\) and \(C\).
Figure~\ref{fig:neuron-decision} highlights two logically equivalent representations of this neuron's Boolean function. Figure~\ref{fig:neuron-obdd} highlights an Ordered Binary Decision Diagram (OBDD) representation\footnote{An Ordered Binary Decision Diagram (OBDD) is a rooted DAG with two sinks: a \(1\)-sink and a \(0\)-sink.  An OBDD is a graphical representation of a Boolean function on variables \(\X = \{X_1,\ldots,X_n\}\). Every OBDD node (but the sinks) is labeled with a variable \(X_i\) and has two labeled outgoing edges: a \(1\)-edge and a \(0\)-edge. The labeling of the OBDD nodes respects a global ordering of the variables \(\X\): if there is an edge from a node labeled \(X_i\) to a node labeled \(X_j\), then \(X_i\) must come before \(X_j\) in the ordering. To evaluate the OBDD on an instance \(\x\), start at the root node of the OBDD and let \(x_i\) be the value of variable \(X_i\) that labels the current node. Repeatedly follow the \(x_i\)-edge of the current node, until a sink node is reached. Reaching the \(1\)-sink means \(\x\) is evaluated to 1 and reaching the \(0\)-sink means \(\x\) is evaluated to 0 by the OBDD.} and Figure~\ref{fig:neuron-circuit} highlights a circuit representation. These functions are equivalent to the sentence:
\[
[ \neg C \wedge (A \vee B) ] \vee [ C \wedge A \wedge B ],
\]
i.e., if \(C\) is \(0\) then \(A\) or \(B\) must be 1 to meet or surpass the threshold \((\ge 0)\), and if \(C\) is 1 then both \(A\) and \(B\) must be 1.

OBDDs, as in Figure~\ref{fig:neuron-obdd}, are \emph{tractable} representations---they support many operations in time \emph{polynomial} (and typically \emph{linear}) in the size of the OBDD \cite{Bryant86,MeinelT98,Wegener00}.  Circuits, as in Figure~\ref{fig:neuron-circuit}, are not in general tractable as OBDDs, although we will later seek to obtain tractable circuits through knowledge compilation, a subject which we will revisit in more depth in Section~\ref{sec:kc}.  Note further that OBDDs are also circuits that are notated more compactly.\footnote{An OBDD node labeled by variable \(X\) and with children \(f_x\) and \(f_{\n(x)}\) is equivalent to the circuit fragment \((x \wedge f_x) \vee (\n(x) \wedge f_{\n(x)}).\)}

Our first goal is to obtain a tractable circuit representation of a given neuron.
First, consider the following class of threshold-based linear classifiers.
\begin{definition}
Let \(\X\) be a set of binary features where each feature \(X\) in \(\X\) has a value \(x \in \{0,1\}\).  Let \(\x\) denote an instantiation of variables \(\X\).  Consider functions \(f\) that map instantiations \(\x\) to a value in \(\{0,1\}\).  We call \(f\) a linear classifier if it has the following form:
\begin{equation} \label{eq:linear}
f(\x) =
\left\{
\begin{tabular}{rl}
1 & if \(\sum_{x \in \x} w_x \cdot x \ge T\) \\
0 & otherwise
\end{tabular}
\right.
\end{equation}
where \(T\) is a threshold, \(x \in \x\) is the value of variable \(X\) in instantiation \(\x\), and where \(w_{x}\) is the real-valued weight associated with value \(x\) of variable \(X\).
\end{definition}
Note that such classifiers are also Boolean functions.
The following result, due to \cite{chanUAI03}, gives us a way of obtaining a tractable circuit representing the Boolean function of such classifiers.
\begin{theorem} \label{theorem:compile}
A linear classifier in the form of Equation~\ref{eq:linear} can be represented by an OBDD of size \(O(2^{\frac{n}{2}})\) nodes, which can be computed in \(O(n 2^{\frac{n}{2}})\) time.
\end{theorem}
\cite{chanUAI03} further provided an algorithm to obtain the result of Theorem~\ref{theorem:compile}, although much more efficiently than what the bounds suggest.  It was originally for compiling naive Bayes classifiers to Ordered Decision Diagrams (ODDs).  However, this algorithm applies to any classifier of the form given by Equation~\ref{eq:linear}, which includes naive Bayes classifiers, but also logistic regression classifiers, as well as neurons with step activation functions \cite{Elkan97}.

Compiling a linear classifier such as a neuron or a naive Bayes classifier is NP-hard \cite{ShihCD18}, hence algorithms, such as the one from \cite{chanUAI03}, are unlikely to have much tighter bounds.  However, we can significantly tighten this bound if we make additional assumptions about the classifier's parameters.
\begin{theorem} \label{theorem:pseudo-compile}
Consider a linear classifier in the form of Equation~\ref{eq:linear}, where the weights \(w_x\) and threshold \(T\) are integers.  Such a classifier can be represented by an OBDD of size \(O(nW)\) nodes, and compiled in \(O(nW)\) time, where \(W = |T| + \sum_x |w_x|\) is a sum of absolute values.
\end{theorem}
\noindent While this result is known, Appendix~\ref{sec:pseduo-compile} provides a construction for completeness.\footnote{This result appears, for example, as an exercise in \url{https://www.cs.ox.ac.uk/people/james.worrell/lectures.html}.  This result also falls as a special case of \cite{ChubarianT20}, which showed how to compile tree-augmented naive Bayes classifiers into OBDDs, where a naive Bayes classifier is a special case.}  Note that the integrality assumption of Theorem~\ref{theorem:pseudo-compile} can be applied to classifiers with real-valued weights by multiplying the parameters by a constant and then truncating (i.e., the parameters have fixed precision). As we show later, this pseudo-polynomial time algorithm enables the compilation of neurons, and ultimately neural networks, with \emph{hundreds} of features. This is in contrast to the preliminary work of \cite{ChoiShiShihDarwiche18}, which was based on the algorithm of \cite{chanUAI03}
for compiling neurons that scaled only to \emph{dozens} of features.

\begin{figure}[tb]
\centering
\begin{subfigure}[b]{.32\linewidth}
\centering
\includegraphics[width=.95\linewidth]{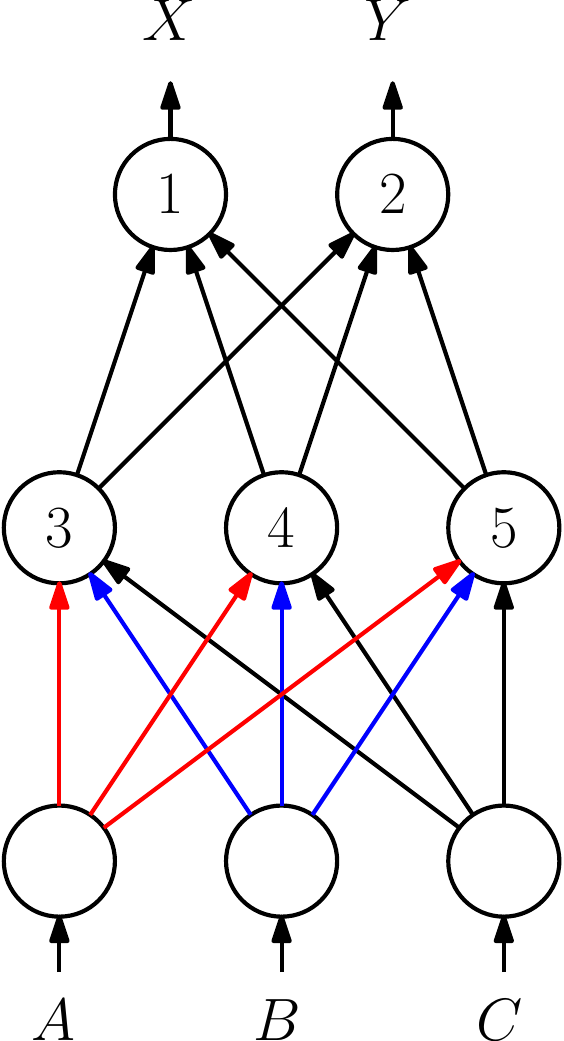}
\caption{neural network}\label{fig:nn}
\end{subfigure}
\begin{subfigure}[b]{.32\linewidth}
\centering
\includegraphics[width=.9\linewidth]{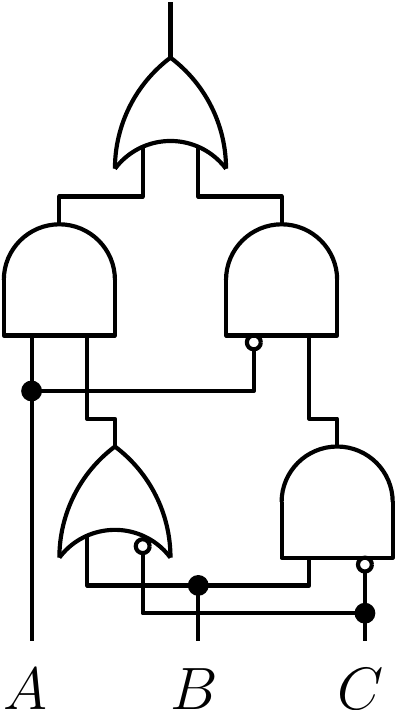}
\caption{circuit of neuron}\label{fig:neuron-circuit-example}
\end{subfigure}
\begin{subfigure}[b]{.32\linewidth}
\centering
\includegraphics[width=.95\linewidth]{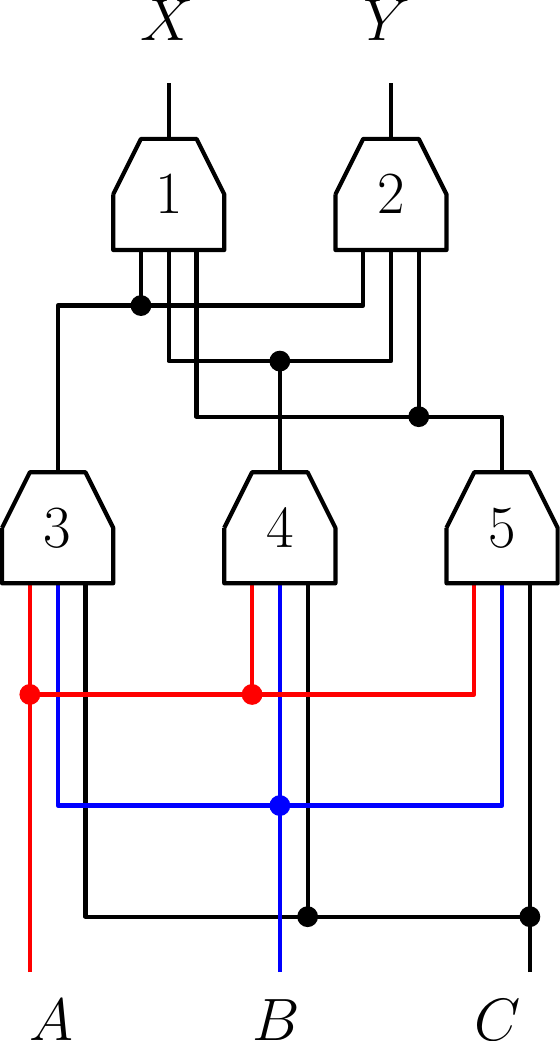}
\caption{circuit of network}\label{fig:nn-circuit}
\end{subfigure}
\caption{A neural network, the circuit of a single neuron, and the circuit of the original network.  Wires highlighted in red and blue correspond to the inputs \(A\) and \(B\), respectively.
 \label{fig:nn-compile}}
\end{figure}

Now that we can compile each neuron into a (tractable) Boolean circuit,  the whole neural network will then induce a Boolean circuit as illustrated in Figure~\ref{fig:nn-compile}. That is, for the given neural network in Figure~\ref{fig:nn}, each neuron is compiled into a Boolean circuit as in Figure~\ref{fig:neuron-circuit-example}. The circuits for neurons are then connected according to the neural network structure, leading to the Boolean circuit in Figure~\ref{fig:nn-circuit}, where the
circuit of each neuron is portrayed as a block. 

Using the algorithm of \cite{chanUAI03}, the Boolean circuit that we obtain from a neuron is tractable.  
The network's Boolean circuit, that we construct from the Boolean circuits of the neurons, may not be tractable however.
To use the explanation and verification techniques proposed in \cite{ShihCD18,ShihCD18b}, we require a tractable circuit; cf. \cite{DarwicheHirth20a}. We next show how to obtain such a circuit using tools from the field of knowledge compilation.

\section{Tractability via Knowledge Compilation} \label{sec:kc}

In this section, we provide a short introduction to the domain of knowledge compilation, and then show how we can compile a neural network into a tractable Boolean circuit.

We follow \cite{darwicheJAIR02}, which considers tractable representations of Boolean circuits, and the trade-offs between succinctness and tractability.  In particular, they consider Boolean circuits of and-gates, or-gates and inverters, but where inverters only appear at the inputs (hence the inputs of the circuit are variables or their negations).  This sub-class of circuits is called \emph{Negation Normal Form} (NNF) circuits. Any circuit with and-gates, or-gates and inverters can be efficiently converted into an NNF circuit while at most doubling its size.

\begin{figure}[tb]
\centering
\includegraphics[width=.5\linewidth]{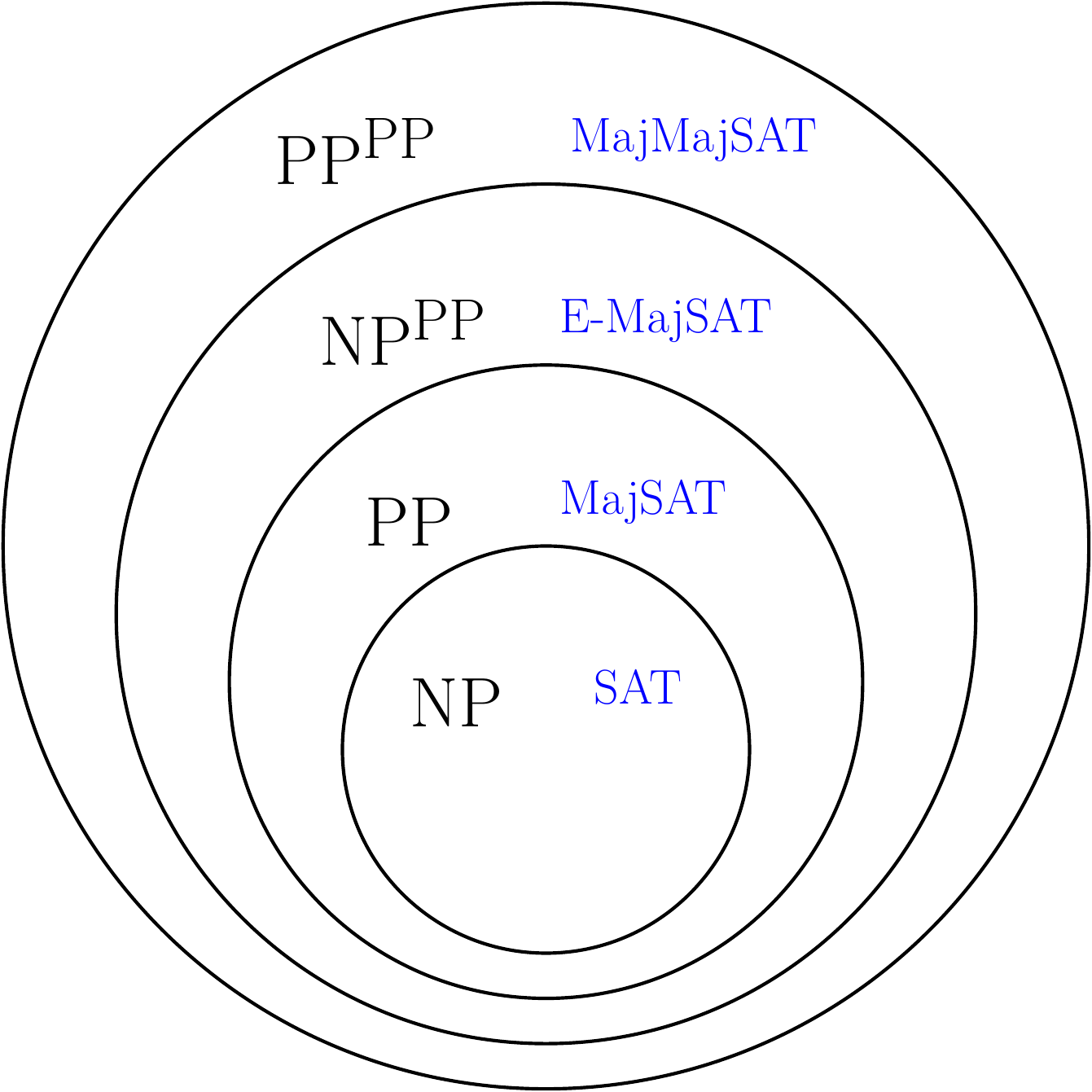}
\caption{Containment of four complexity classes: \(\mathrm{NP} \subseteq \mathrm{PP} \subseteq \mathrm{NP}^\mathrm{PP} \subseteq \mathrm{PP}^\mathrm{PP}\).  Their canonical problems are labeled in blue.
 \label{fig:complexity}}
\end{figure}

By imposing properties on the structure of NNF circuits, one can obtain greater tractability (the ability to perform certain operations in polytime) at the possible expense of succinctness (the size of the resulting circuit).  To motivate this trade-off, consider Figure~\ref{fig:complexity}, which highlights the containment relationship between four complexity classes.  The ``easiest'' class is \(\mathrm{NP}\), and the ``hardest'' class is \(\mathrm{PP}^\mathrm{PP}\).  The canonical problems that are complete for each class all correspond to queries on Boolean expressions.  One popular computational paradigm for solving problems in these classes is to reduce them to the canonical problem for that class, and to compile the resulting Boolean expressions to circuits with the appropriate properties.\footnote{For more on this paradigm, see \url{http://beyondnp.org}.}\footnote{For a video tutorial on this paradigm, ``On the role of logic in probabilistic inference and machine learning,'' see  \url{https://www.youtube.com/watch?v=xRxP2Wj4kuA}} 
For example, \cite{OztokCD16} shows how to solve
\(\mathrm{PP}^\mathrm{PP}\)-complete problems by reduction to MajMajSAT queries on a specific tractable class of Boolean circuits.

Consider now a property on NNF circuits called \emph{decomposability} \cite{darwicheJACM-DNNF}.  This property asserts that the sub-circuits feeding into an and-gate cannot share variables.  An NNF circuit that is decomposable is said to be in Decomposable Negation Normal Form (DNNF).  In a DNNF circuit, testing whether the circuit is satisfiable can be done in time \emph{linear} in the size of the circuit.  Another such property is \emph{determinism} \cite{darwiche01tractable}.  This property asserts that for each or-gate, if the or-gate outputs 1 then exactly one of its input is 1.  A DNNF circuit that is also deterministic is called a d-DNNF.  The circuit in Figure~\ref{fig:neuron-circuit} is an example of a d-DNNF circuit.  In a d-DNNF circuit, counting the number of assignments that satisfy the circuit can be done in time \emph{linear} in the size of the circuit, assuming the circuit also satisfies smoothness \cite{Darwiche03}.\footnote{Counting how many assignments satisfy a given circuit allows us to tell whether a majority of them satisfy the circuit (MajSAT).}  Hence, with these first two properties, we can solve the canonical problems in the two ``easiest'' classes illustrated in Figure~\ref{fig:complexity}.

A more recently proposed class of circuits is the Sentential Decision Diagram (SDD) \cite{Darwiche11,XueChoiDarwiche12,ChoiDarwiche13}.  SDDs are a subclass of d-DNNF circuits that assert a stronger form of decomposability, and a stronger form of determinism.  SDDs subsume OBDDs and are exponentially more succinct \cite{Bova16}.   SDDs support polytime conjunction and disjunction.  That is, given two SDDs \(\alpha\) and \(\beta\), there is a polytime algorithm to construct another SDD \(\gamma\) that represents \(\alpha \wedge \beta\) or \(\alpha \vee \beta\).\footnote{If \(s\) and \(t\) are the sizes of input SDDs, then conjoining or disjoining the SDDs takes \(O(s \cdot t)\) time, although the resulting SDD may not be compressed \cite{VdBDarwiche15}.}  Further, SDDs can be negated in linear time.\footnote{In our case study in Section~\ref{sec:case-study}, we used the open-source SDD package available at \url{http://reasoning.cs.ucla.edu/sdd/}.}

These polytime operations allow a simple algorithm for compiling a Boolean circuit with and-gates, or-gates and inverters into an SDD. We first obtain an SDD for each circuit input. We then traverse the circuit bottom-up, compiling the output of each visited gate into an SDD by applying the corresponding operation to the SDDs of the gate's inputs.

SAT and MajSAT can be solved in linear time on SDDs. Further properties on SDDs allow the problems E-MajSAT and MajMajSAT, the two hardest problems illustrated in Figure~\ref{fig:complexity}, to be also solved in time linear in the size of the SDD \cite{OztokCD16}. In our experiments, we compiled the Boolean circuits of neural networks into standard SDDs as this was sufficient for efficiently supporting the explanation and verification queries we are interested in.

\section{On the Robustness of Classifiers} \label{sec:robustness}

Neural networks are now ubiquitous in machine learning and artificial intelligence, but it is increasingly apparent that neural networks learned in practice can be fragile.  That is, they can be susceptible to misclassifying an instance after small perturbations have been applied to it \cite{SzegedyZSBEGF13,GoodfellowSS14,Dezfooli16,WangPWYJ18,ZhangZH19}.  Next, we show how compiling a neural network into a tractable circuit can provide one with the ability to analyze the robustness of a neural network's decisions.

We consider first the robustness of a binary classifier's decision to label a given instance \(0\) or \(1\).  The notion of robustness that we consider is based on the following question: how many features do we need to flip from \(0\) to \(1\) or \(1\) to \(0\), before the classifier's decision flips?  That is, we consider the robustness of a given instance to be the (Hamming) distance to the closest instance of the opposite label.

\begin{definition}[Instance-Based Robustness]
Consider a Boolean classification function \(f : \{0,1\}^n \to \{0,1\}\) and a given instance \(\x.\)  The \underline{robustness} of the classification of \(\x\) by \(f\), denoted by \(\robust_f(\x)\), is defined as follows.  If \(f\) is a trivial function (true or false), then \(\robust_f(\x) = \infty\).  Otherwise,
\[
\robust_f(\x) = \min_{\x': f(\x') \neq f(\x)} d(\x, \x')
\]
where \(d(\x', \x)\) denotes the Hamming distance between \(\x'\) and \(\x\), i.e., the number of variables on which \(\x\) and \(\x'\) differ.
\end{definition}
This notion of robustness was also considered by \cite{ShihCD18b}, who also assumed binary (or categorical) features.  Other notions of robustness are discussed in \cite{Leofante18}, for real-valued features.

Given a classification function \(f\), we refer to an instance \(\x\) as being \(k\)-robust if \(\robust_f(\x) = k\), i.e., it takes at least \(k\) flips of the features to flip the classification.
In general, it is intractable to compute the robustness of a classification, unless P=NP.  Consider the following decision problem:
\begin{quote}
\textbf{D-ROBUST:} Given function \(f\), instance \(\x\), and integer \(k\), is \(\robust_f(\x) \ge k\)?
\end{quote}
\begin{theorem} \label{thm:co-np}
\emph{\textbf{D-ROBUST}} is coNP-complete.
\end{theorem}

\begin{proof}
Let \(f\) be a DNF formula over \(n\) variables, and let \(\x\) be an arbitrary instantiation of these variables.  Further, let \(\ell\) be a literal of a new variable.  The problem is in coNP, because it is polytime falsifiable: given a counterexample \(\x^\prime,\) we can check in polytime that \(D(\x,\x^\prime) < k\) and that \(f(\x^\prime) \ne f(\x).\) The problem is coNP-hard since \(f\) is a tautology iff the DNF \(f \vee \ell\) and instantiation \(\x\) has robustness at least \(n+1\) (i.e., infinite).  See also Footnote~\ref{footnote:fn}.
\end{proof}
Theorem~\ref{thm:co-np} implies that we can use a SAT solver to compute the robustness of an instance relative to a Boolean function.\footnote{For a given \(k\), we can determine if \(r_f(x) \leq k\) by first encoding the set of instances within a distance of \(k\) away from \(x\) as a CNF formula \(\phi\), using a standard encoding. We then encode the value of \(f(x)\) and the neural network's classification function as a CNF formula \(\Delta\), using a technique similar to that by~\cite{NarodytskaKRSW18}.  The formula \(\phi \wedge \Delta\) is then satisfiable iff \(r_f(x) \leq k\). Finally, we iterate over all possible values of \(k\) (or perform binary search), so the instance-based robustness is just the smallest value of \(k\) such that \(r_f(x) \leq k\) (i.e., \(\phi \wedge \Delta\) is SAT).}  Given a tractable circuit (in particular, an OBDD), this question can also be answered in time \emph{linear} in the size of the circuit \cite{ShihCD18b}.\footnote{The robustness of an instance \(y,\x\) can be computed by the following recurrence, which recurses on the structure of an OBDD: 
\(\robust_f(y,\x) = \min\{ \robust_{f|y}(\x), 1+\robust_{f|\bar{y}}(\x)\}\)
where \(\robust_f(\x) = 0\) if \(f\) is false and \(\robust_f(\x) = \infty\) if \(f\) is true.}  We employ the algorithm given by \cite{ShihCD18b} in our case studies in Section~\ref{sec:case-study}.

Next, rather than consider the robustness of just one classification, we can consider the \emph{average} robustness of a classification function, over all possible inputs.  In other words, we consider the \emph{expected} robustness of a classifier, under a uniform distribution of its inputs.

\begin{definition}[Model-based Robustness]
Consider a Boolean classification function \(f : \{0,1\}^n \rightarrow \{0,1\}\). The \underline{model robustness} of \(f\) is defined as:
\[
\modelrobust(f)
= \frac{1}{2^n} \sum_{\x} \robust_f(\x)
\] 
\end{definition}

\begin{algorithm}[t]
    \caption{\texttt{model-robustness}($f$)}\label{alg:model}
    \Ainput{A classifier's Boolean function $f$}

    \Aoutput{The (positive) model-based robustness $\modelrobust(f)$}

    \Amain{
    \begin{algorithmic}[1]{
    \STATE $M,\gek(1) \gets 0,f$
    \FOR{$k$ from $2$ to $n$}   
        \STATE $\gek(k) \gets \bigwedge_X ( \gek(k-1)|x \wedge \gek(k-1)|\n(x) )$
        \STATE $f_{k-1} \gets \gek(k-1) \wedge \neg \gek(k)$
        \STATE $M \gets M + (k-1) \cdot \modelcount(f_{k-1})$
    \ENDFOR
    \STATE $M \gets M + n \cdot \modelcount(\gek(n))$ \hfill\COMMENT{since $f_n \equiv \gek(n)$}
    \STATE \textbf{return} $M$}
    \end{algorithmic}
    }
\end{algorithm}

Let $f$ be the classification function whose robustness we want to assess.  If \(f(\x) = 1\) we refer to \(\x\) as a positive instance; otherwise \(f(\x) = 0\) and we refer to \(\x\) as a negative instance.  We propose Algorithm~\ref{alg:model} for computing the model robustness of a classifier over all \emph{positive} instances (the robustness of negative instances can be computed by invoking Algorithm~\ref{alg:model} on function \(\neg f\)).\footnote{Recently, \cite{BalutaSSMS19} proposed an approach for estimating robustness using approximate model-counting, with PAC-style guarantees.  Their approach scaled to \(10 \times 10\) digits datasets; our exact approach scales to \(16 \times 16\) digits datasets in Section~\ref{sec:case-study}.}

Our algorithm is based on computing the set of functions \(\gek(k)\), which are the Boolean functions representing all positive instances \(\x\) that have robustness \(k\) or higher.  That is, \(\gek(k)\) represents all instances \(\x\) where \(f(\x) = 1\) and where \(\robust_f(\x) \ge k\).
First, \(\gek(1) = f\).  For \(k=2\) we have:
\(
\gek(2)
= \bigwedge_X ( f|x \wedge f|\n(x) ),
\)
where \(f|x\) denotes the conditioning of \(f\) on value \(x\), i.e., the function that we would obtain by setting \(X\) to true (replace every occurrence of \(X\) with true, and in the case of \(f|\n(x),\) replace \(X\) with false).
Say that \(\x\) is an instance of $f|x \wedge f|\n(x)$, and thus \(f(\x) = 1\) and \(f(\x)\) remains \(1\) no matter how we set \(X\).  By taking the conjunction across all variables \(X\), we obtain all instances \(\x\) whose output would not flip after flipping any single feature \(X\).  Next, consider the robustness of the instances of \(\gek(2).\)  Some of these instances \(\x\) will become \(1\)-robust with respect to \(\gek(2).\)  These instances are in turn \(2\)-robust with respect to the original function \(f\).
More generally, we can compute \(\gek(k)\) from \(\gek(k-1)\), via
\(
\gek(k)
= \bigwedge_X ( \gek(k-1)|x \wedge \gek(k-1)|\n(x) ).
\)
We can now compute the functions \(f_k\) representing all of the \(k\)-robust examples of \(f\), via
\(
f_k = \gek(k) \wedge \neg \gek(k+1).
\)
The model count of \(f\), denoted by \(\modelcount(f)\), is the number of instances \(\x\) satisfying a Boolean function \(f\).  We can then compute the model robustness of \(f\) by:
\[
\modelrobust(f) = \frac{1}{2^n} \sum_{k=1}^{n} \modelcount(f_k) \cdot k.
\]

Consider now the \emph{most} robust instances of a function \(f\).
\begin{definition}[Maximum Robustness]
Consider a Boolean classification function \(f : \{0,1\}^n \rightarrow \{0,1\},\) where \(f\) is non-trivial.  The \underline{maximum robustness} of \(f\) is defined as:
\[
\maxrobust(f) = \max_{\x} \robust_f(\x).
\]
\end{definition}
Note that the instances \(\x\) of \(\gek(k)\) is a subset of the instances of \(\gek(k-1)\), as computed in Algorithm~\ref{alg:model}.  Hence, the model count of \(\gek(k)\) will decrease as we increase \(k\).  At a large enough \(k\), then \(\gek(k+1),\) and hence also \(f_{k+1}\), will have no models and will equal false.  At this point, we know \(k\) is the maximum robustness, and we can stop Algorithm~\ref{alg:model} early.
Further, this \(f_k\) gives us the set of examples that are most robust (requiring the most number of features to flip).\footnote{Note that if \(f\) is a non-trivial Boolean function, then \(f_n\) must be false.  Suppose \(f_n\) were not false, and that \(\x\) is an instance of \(f_n\).  This means we can flip any and all variables of \(\x\), and it would always be an example of \(f\).  This implies that \(f\) must have been true, and hence a trivial function. \label{footnote:fn}}

Finally, we observe that model-based robustness appears to be computationally more difficult than instance-based robustness.  In particular, the model-robustness over positive instances can be shown to be a PP-hard problem.
\begin{quote}
\textbf{D-P-MODEL-ROBUST:} Given function \(f\) and integer \(k\), is the positive model robustness of \(f\) at least \(k\)?
\end{quote}
\begin{theorem}
\emph{\textbf{D-P-MODEL-ROBUST}} is PP-hard.
\end{theorem}
\begin{proof}
Let \(f\) be a CNF formula over \(n\) variables and let \(\ell\) be a literal of a new variable.  Note that model counting is \#P-complete for CNF \cite{Valiant79}, and that any positive instance of the CNF \(f \wedge \ell\)  has a robustness of \(1\). It follows that \(\modelcount(f) \ge k\) iff the positive model robustness of \(f \wedge \ell\) is at least \({k}/{2^{n+1}}\).
\end{proof}
To compute model-based robustness using Algorithm~\ref{alg:model}, we must be able to negate, conjoin and condition on Boolean functions, as well as compute their model count.  Given a circuit represented as an SDD, operations such as negation and counting the models of an SDD can be done in time \emph{linear} in the size of the SDD.  Conjoining two SDDs of size \(s\) and \(t\) takes time \(O(st)\) although a sequence of conjoin operations may still take exponential time, as in Algorithm~\ref{alg:model}.

\section{A Case Study} \label{sec:case-study}

\def\digitzero{digit-0}
\def\digitone{digit-1}

We next provide a case study in explaining and verifying a convolutional neural network via knowledge compilation.

\subsection{(Binary) Convolutional Neural Networks}

In our case study, we consider \emph{binary} convolutional neural networks (binary CNNs).\footnote{A number of binary variations of neural networks have been proposed in the literature.  The XNOR-Networks of \cite{RastegariORF16} are another binary variation of CNNs, which also assumes binary weights.  The binarized neural networks (BNNs) of \cite{hubara2016binarized} have binarized parameters and activations.  In work closely related to ours, \cite{NarodytskaKRSW18} studied the verification of BNNs, using SAT solvers, as discussed in Section~\ref{sec:intro}.}  That is, if we assume binary inputs and step activations, then the outputs of all neurons are binary, and the output of the network itself is also binary.  Hence a binary CNN represents a Boolean function.  We can construct a Boolean circuit representing a binary CNN, and then compile it to a tractable one as described in Section~\ref{sec:kc}.

Our binary CNNs contain three types of layers:
\begin{itemize}
\item \textit{convolution + step layers:} a convolution layer consists of a set of filters, that can be used to detect local patterns in an input image.  Typically, a ReLU unit is applied to the output of a filter.  In a binary CNN, we assume step activations (whose parameters are trained first using sigmoid activations, then replacing them with step activations);
\item \textit{max-pooling layers:} a max-pooling layer can be used to reduce the dimension of an image, helping to reduce the overall computational and statistical demands.  In a binary CNN, if the inputs of a max-pooling layer is \(0\) or \(1\), then the ``max'' reduces to a logical ``or'';
\item \textit{fully-connected layers:} if the inputs are binary and if we use step activations, then each neuron represents a Boolean function, as in Section~\ref{sec:neurons}.
\end{itemize}

\subsection{Experimental Setup}

We consider the USPS digits dataset of handwritten digits, consisting of \(16 \times 16\) pixel images, which we binarized to black and white \cite{hull1994database}.  Here, we performed binary classification using different pairs of digits.   We first trained a CNN using sigmoid activations, using TensorFlow.  We replaced the sigmoid activations with step activations, to obtain a binary CNN that we compiled into a tractable circuit.  In particular, we compiled the binary CNN into a Sentential Decision Diagrams (SDD).  We shall subsequently provide analyses of the binary CNN, via queries on the SDD.

More specifically, we created two convolution layers, each with stride size $2$. We first swept a $3 \times 3$ filter on the original \(16 \times 16\) image (resulting in a $7 \times 7$ grid), followed by a second \(2 \times 2\) filter (resulting in a $3 \times 3$ grid).  These outputs were the inputs of a fully-connected layer with a single output.  We did not use max-pooling as the dimension was reduced enough by the convolutions.  Finally, we optimized a sigmoid cross-entropy loss using the Adam optimizer. 

The SDD circuits compiled from neural networks in the following experiments are exact as the compilation process utilized the exact neuron compiler based on~\cite{chanUAI03}.\footnote{An updated version of the compiler used in these experiments is available at \url{https://github.com/art-ai/nnf2sdd}. The compiler integrates the newly proposed neuron compiler with pseudo-polynomial time complexity.} We later also evaluate the newly proposed, approximate neuron compiler, which has a pseudo-polynomial time complexity, showing the trade-off it leads to between classification accuracy and size of compilation.

\subsection{Explaining Decisions}

\setlength{\fboxsep}{0pt}
\setlength{\fboxrule}{0pt}
\definecolor{mygray}{gray}{0.6}
\begin{figure}[t]
    \centering
    \begin{subfigure}[b]{0.45\linewidth} 
      \centering
      \captionsetup{width=\linewidth}
      \fcolorbox{mygray}{mygray}{\includegraphics[width=0.7\linewidth]{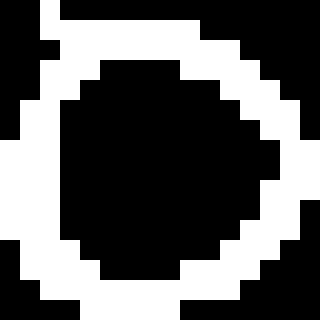}}
      \caption{\digitzero} \label{fig:inst_0}
    \end{subfigure}
    \begin{subfigure}[b]{0.45\linewidth} 
      \centering
      \captionsetup{width=\linewidth}
      \fcolorbox{mygray}{mygray}{\includegraphics[width=0.7\linewidth]{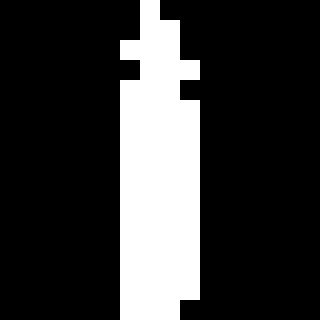}}
      \caption{\digitone} \label{fig:inst_1}
    \end{subfigure}
    \\
    \begin{subfigure}[b]{0.45\linewidth} 
      \centering
      \captionsetup{width=\linewidth}
      \fcolorbox{mygray}{mygray}{\includegraphics[width=0.7\linewidth]{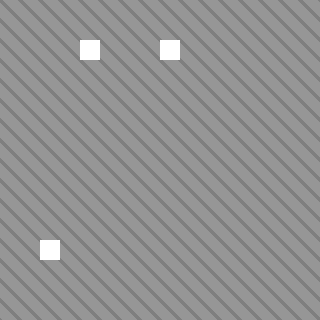}}
      \caption{\digitzero} \label{fig:pi_inst_0}
    \end{subfigure}
    \begin{subfigure}[b]{0.45\linewidth} 
      \centering
      \captionsetup{width=\linewidth}
      \fcolorbox{mygray}{mygray}{\includegraphics[width=0.7\linewidth]{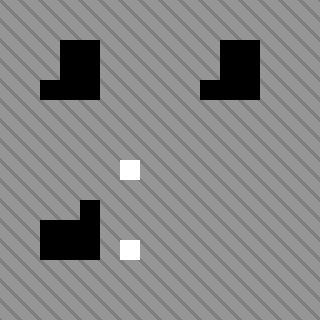}}
      \caption{\digitone} \label{fig:pi_inst_1}
    \end{subfigure}
    \begin{subfigure}[b]{0.45\linewidth} 
      \centering
      \captionsetup{width=0.8\linewidth}
      \fcolorbox{mygray}{mygray}{\includegraphics[width=0.7\linewidth]{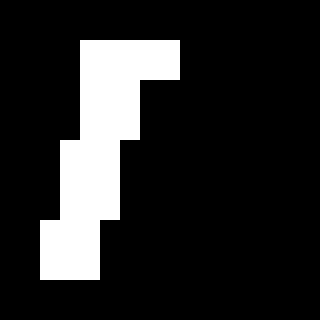}}
      \caption{a 1 labeled as a 0} \label{fig:pi_fool_0}
    \end{subfigure}
    \begin{subfigure}[b]{0.45\linewidth} 
      \centering
      \captionsetup{width=0.8\linewidth}
      \fcolorbox{mygray}{mygray}{\includegraphics[width=0.7\linewidth]{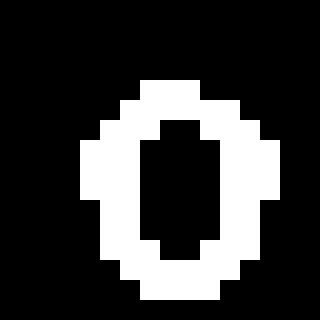}}
      \caption{a 0 labeled as a 1} \label{fig:pi_fool_1}
    \end{subfigure}
\caption{A correctly classified \digitzero\ and \digitone\ from the test set (first row), the corresponding PI-explanations (second row, gray striped regions represent ``don't care's''), and the corresponding fooling images (third row).
} \label{fig:inst_digits}
\end{figure}

We consider how to explain \emph{why} a neural network classified a given instance positively or negatively. In particular, we consider \emph{prime-implicant explanations} (PI-explanations), as proposed by \cite{ShihCD18}; see also \cite{IgnatievNM19a,DarwicheHirth20a}.
Say that an input image \(\x\) is classified positively, i.e., as a~\digitone. A PI-explanation returns the smallest subset \(\y\) of the inputs in \(\x\) that render the remaining inputs irrelevant. That is, once you fix the pixel values \(\y\), the values of the other pixels do not matter---the network will always classify the instance as a~\digitone.\footnote{Consider in contrast ``Anchors,'' recently proposed by \cite{anchor:nipsws16,anchors:aaai18}. An anchor for an instance \(\x\) is a subset of the instance that is highly likely to be classified with the same label, no matter how the missing features are filled in (according to some distribution).  In contrast, PI-explanations are \emph{exact}.}

We first trained a CNN to distinguish between \digitzero\ and \digitone\ images, which achieved \(98.74\%\) accuracy.  The resulting SDD had 5,900 nodes and 28,735 edges.
We took one correctly classified instance of each digit from the test set, shown in Figures~\ref{fig:inst_0}~\&~\ref{fig:inst_1}. The shortest PI-explanations for these two images are displayed in Figures~\ref{fig:pi_inst_0}~\&~\ref{fig:pi_inst_1}.
In Figure~\ref{fig:pi_inst_0}, the PI-explanation consists of three white pixels. Once we fix these three pixels, the network will always classify the image as a~\digitzero, no matter how the pixels in the gray region are set. Similarly, Figure~\ref{fig:pi_inst_1} sets three black patches of pixels to the left and right, and sets two center pixels to white, which is sufficient for the network to always classify the image as a~\digitone.

These PI-explanations provide strong guarantees: the pixels in the gray region can be manipulated in \emph{any} way and the classification would still not change. They are so strong in fact that one can easily create counterexamples to fool the network. In Figures~\ref{fig:pi_fool_0}~\&~\ref{fig:pi_fool_1}, we fill in the remaining pixels in such a way that the~\digitzero~image looks like a~\digitone, and vice versa. The network classifies these new images incorrectly because it is misled by the subset of pixels shown in the PI-explanation of Figures~\ref{fig:pi_inst_0}~\&~\ref{fig:pi_inst_1}. Using this method, we can readily generate counterexamples such as these.
We obtained similar results with other pairs of digits.

\subsection{Explaining Model Behavior}

\begin{figure}[t]
    \centering
    \begin{subfigure}[b]{0.45\linewidth} 
      \centering
      \captionsetup{width=\linewidth}
      \fcolorbox{mygray}{mygray}{\includegraphics[width=0.7\linewidth]{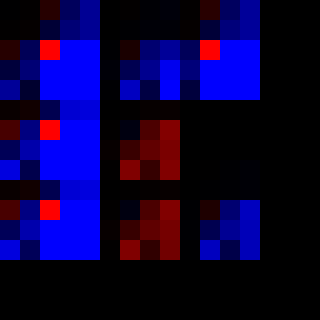}}
      \caption{Marginal grid} \label{fig:marginal_grid}
    \end{subfigure}
    \begin{subfigure}[b]{0.45\linewidth} 
      \centering
      \captionsetup{width=\linewidth}
      \fcolorbox{mygray}{mygray}{\includegraphics[width=0.7\linewidth]{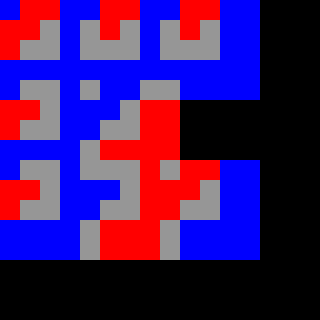}}
      \caption{Unateness grid} \label{fig:unate_grid}
    \end{subfigure}
\caption{Visualizations of the relationship between the output of the network and each individual pixel.
} \label{fig:grids}
\end{figure}

To explain the network's behavior as a whole (and not just per instance), we provide two visualizations of how each pixel contributes to the classification decision: a \emph{marginal grid} and a \emph{unateness grid}.
Figure~\ref{fig:marginal_grid} is a \emph{marginal grid}, which highlights the marginals of the output neuron, i.e., the probability that each pixel is white given that the output of the network is~\digitone.  In general, it is intractable to compute such marginals, which naively entails enumerating all \(2^{256}\) possible input images and then checking the network output.  If we can compile a neural network's Boolean function into a tractable circuit, like an SDD, then we can compute such marginals in time linear in the size of the circuit.

In Figure~\ref{fig:marginal_grid}, red pixels correspond to marginals greater than \(\frac{1}{2}\), and redder pixels are closer to one. Blue pixels correspond to marginals less than \(\frac{1}{2}\), and bluer pixels are closer to zero. The grid intensities have been re-scaled for clarity. Not surprisingly, we find that if the output of the network is high (indicative of a digit-1), then it is somewhat more likely that the pixels in the middle are set to white.

Figure~\ref{fig:unate_grid} is a \emph{unateness grid}, which identifies pixels that sway the classification in one direction only. Red pixels are positively unate (monotone), so turning them from off to on can only flip the classification from~\digitzero~to~\digitone. Blue pixels are negatively unate, i.e., turning them from off to on can only flip the classification from~\digitone~to~\digitzero. Black pixels are ignored by the network completely.
Finally, gray pixels do not satisfy any unateness property.  
In general, determining whether an input of a Boolean function is unate/monotone or unused are computationally hard problems.  In tractable circuits such as SDDs, they are queries that can be performed in time polynomial in the circuit size. 

In Figure~\ref{fig:unate_grid}, the majority of pixels are unate (monotone), suggesting that the overall network behavior is still relatively simple.  Note that there are many unused pixels on the right and bottom borders.  This can be explained by the lack of padding (i.e., given the filter size and stride length, no filter takes any of these pixels as inputs).  There is another block of unused pixels closer to the middle.  On closer inspection, these pixels are unique to one particular filter in the second convolution layer (no other filter depends on their values).  In the tractable circuit of the output neuron, we find that the circuit does not essentially depend on the output of this filter.  Thus, the output of the network does not depend on the values of any of these pixels.  Note that deciding whether an input of a neuron is unused is an NP-hard problem.\footnote{This reduction is similar to the one showing that compiling a linear classifier is NP-hard \cite{ShihCD18}.} However, given a tractable circuit such as an SDD, this question can be answered in time linear in the size of the circuit.

We emphasize a few points now. First, this (visual) analysis is enabled by the tractability of the circuit, which allows marginals to be computed and unate pixels to be identified efficiently. Second, the analysis also emphasizes that the network is not learning the conceptual differences between a~\digitzero~and a~\digitone. It is identifying subsets of the pixels that best differentiate between images of~\digitzero~and~\digitone~from the training set with high accuracy. This perhaps explains why it is sometimes easy to ``fool'' neural networks, which we demonstrated in Figure~\ref{fig:inst_digits}.

\subsection{Analyzing Classifier Robustness} \label{sec:robust-case-study}

\begin{figure}[t]
\centering
  \includegraphics[width=.8\linewidth]{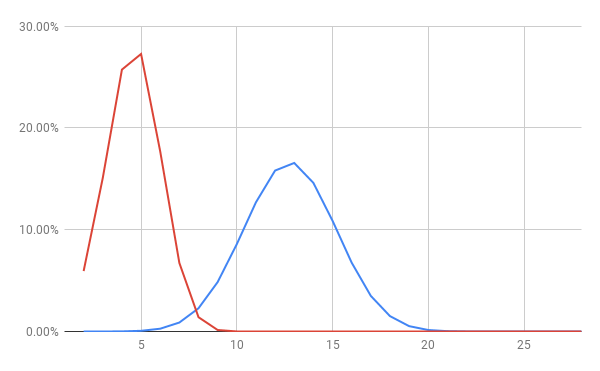}
  \label{fig:model_robustness}
\caption{Level of robustness \(k\) vs. proportion of instances.  Net 1 is plotted in blue (right) and Net 2 in red (left).}
\label{fig:robust}
\end{figure}

Next, we provide a case study in analyzing CNNs based on their robustness. We consider the classification task of discriminating between a digit-1 and a digit-2. 
%
First, we trained two CNNs with the same architectures (as described earlier), but using two different parameter seeds.  We achieved 98.18\% (Net 1) and 96.93\% (Net 2) testing accuracies.  
The SDD of Net 1 had 1,298 nodes and a size of 3,653.  The SDD of Net 2 had 203 nodes and a size of 440.\footnote{The size of a decision node in an SDD is the number of its children.  The size of an SDD is the aggregate size of its nodes.}
Net 1 obtained a model-robustness of 11.77 but Net 2 only obtained a robustness of 3.62.  For Net 2, this means that on average, 3.62 pixel flips are needed to flip a digit-1 classification to digit-2, or vice versa.  Moreover, the maximum-robustness of the Net 1 was 27, while that of Net 2 was only 13.  For Net 1, this means that there is an instance that would not flip unless you flipped (the right) 27 pixels.  These are two networks which are similar in terms of accuracy (differing by only 1.25\%), but \emph{very} different when compared by robustness.  

Figure~\ref{fig:robust} further highlights the differences between these two networks by the level of robustness \(k\).  On the \(x\)-axis, we increase the level of robustness \(k\) (up to the max of 27), and on the \(y\)-axis we measure the proportion of instances with robustness \(k\), i.e., we plot \(2^{-256} \cdot \modelcount(f_k)\), as in Section~\ref{sec:robustness}.
Clearly, the first network more robustly classifies a larger number of instances.
Given two networks with similar accuracies, we prefer the one that is more robust, as it would be more resilient to adversarial perturbations and to noise.  When we compute the average instance-based robustness of testing instances, Net 1 obtains an average of 4.47, whereas Net 2 obtains a lower average of 2.61, as expected.  

\def\myfigwidth{0.45\linewidth}
\def\myplotwidth{0.7\linewidth}

\begin{figure}[t]
\centering
\begin{subfigure}{\myfigwidth}
  \centering
  \includegraphics[width=\myplotwidth]{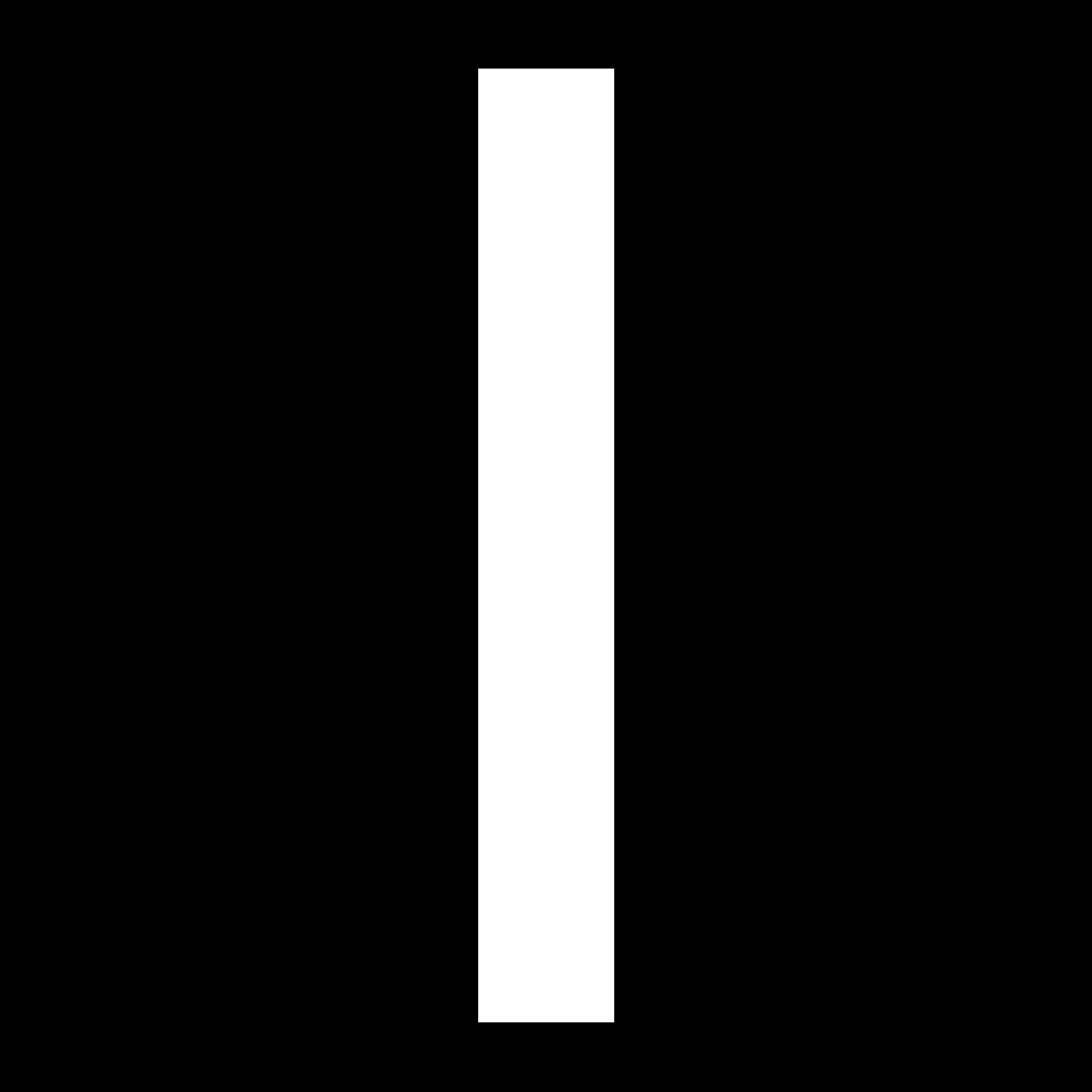}
  \caption{a most robust digit-1}\label{fig:sub1}
\end{subfigure}%
\begin{subfigure}{\myfigwidth}
 \centering
  \includegraphics[width=\myplotwidth]{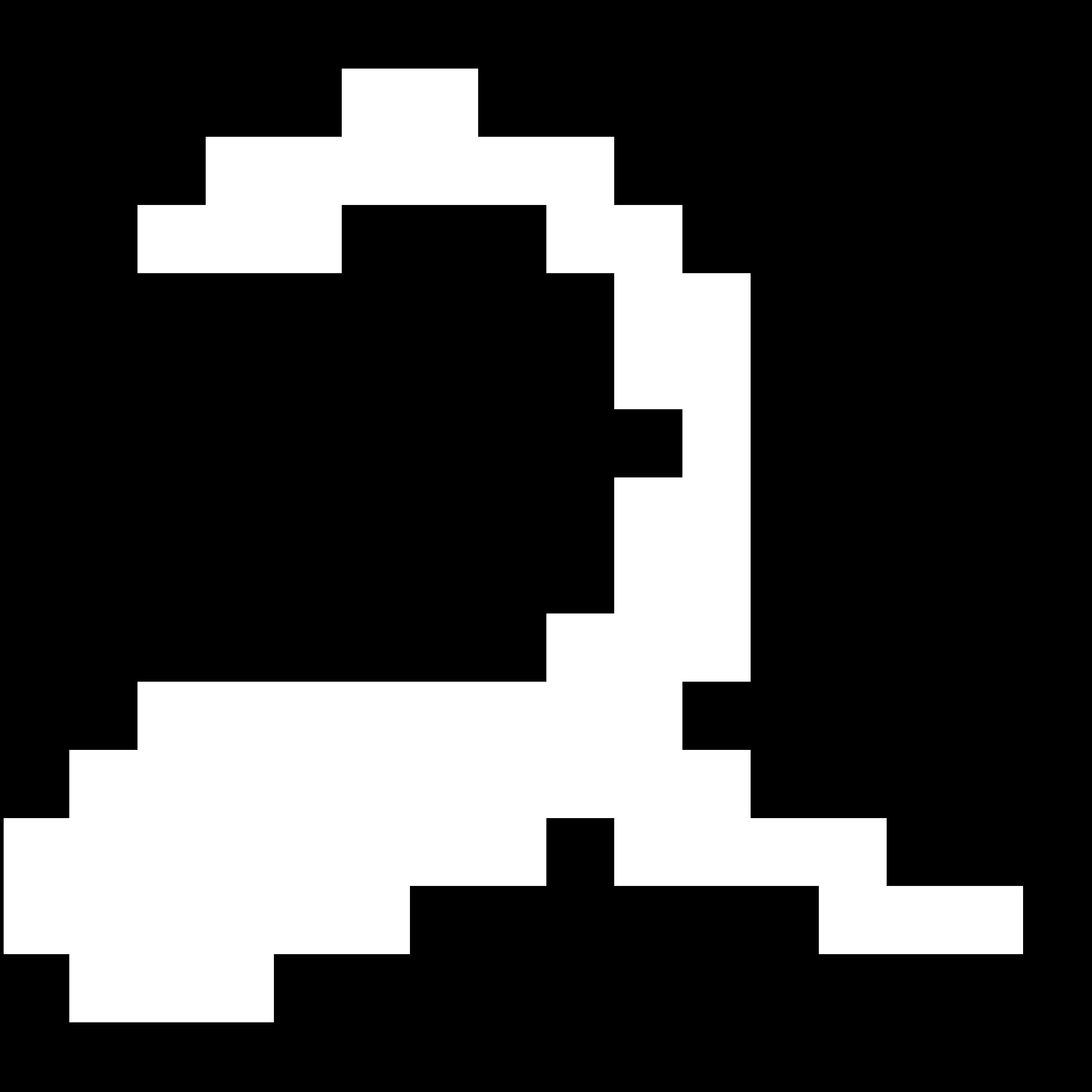}
  \caption{a most robust digit-2}\label{fig:sub8}
 \end{subfigure}%
\\
\begin{subfigure}{\myfigwidth}
\centering
  \includegraphics[width=\myplotwidth]{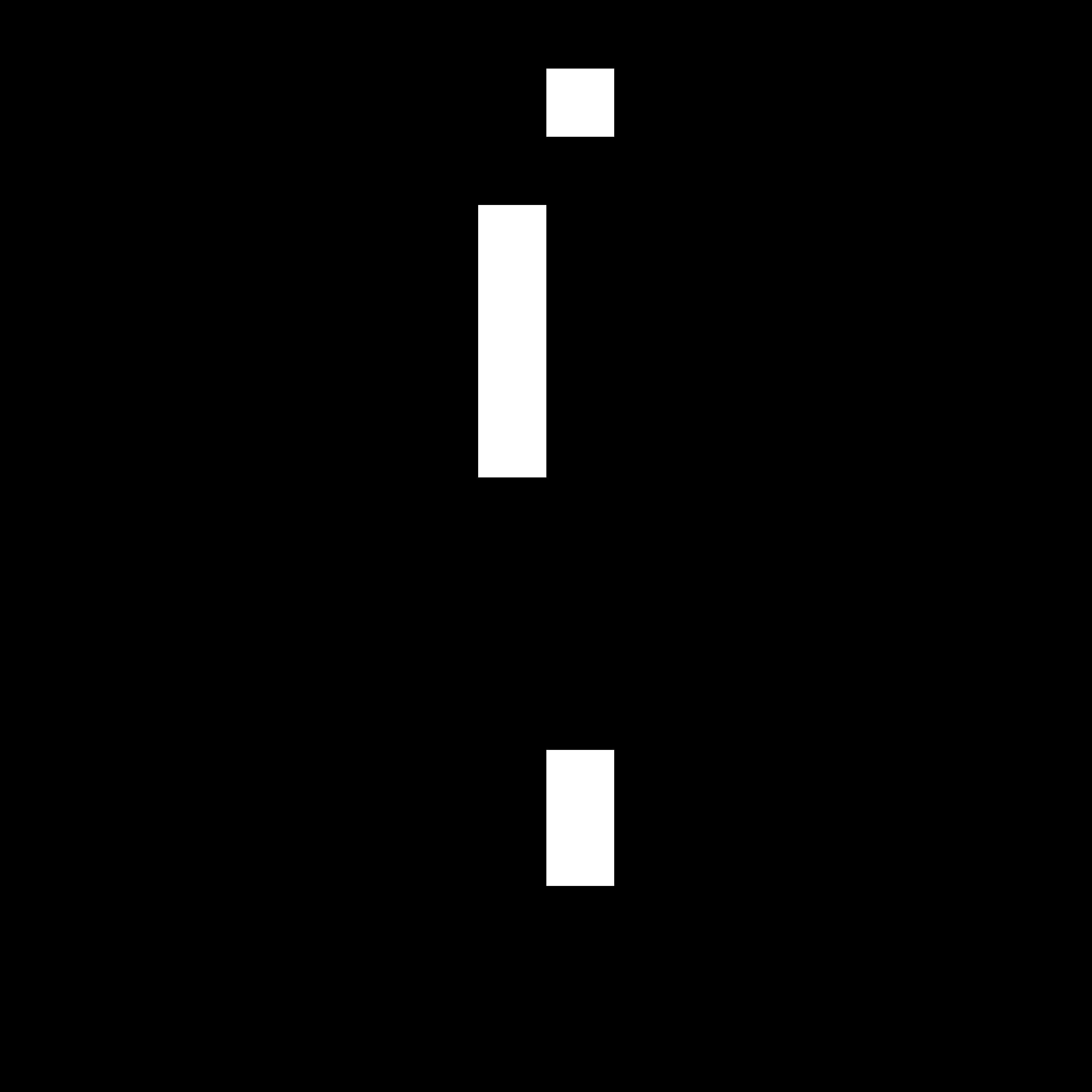}
  \caption{a least robust digit-1}\label{fig:sub3}
\end{subfigure}%
\begin{subfigure}{\myfigwidth}
  \centering
  \includegraphics[width=\myplotwidth]{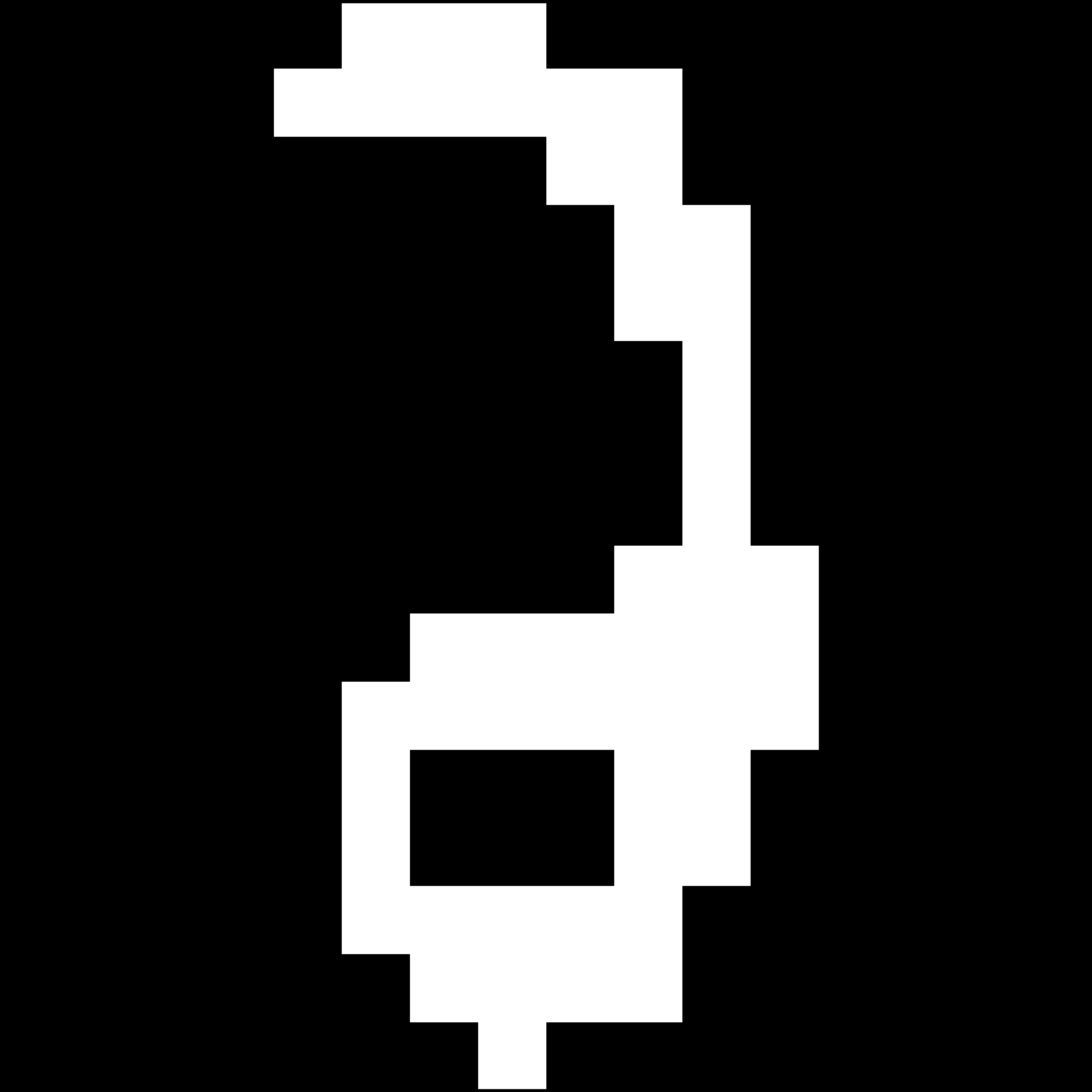}
  \caption{a least robust digit-2}\label{fig:sub4}
 \end{subfigure}%
\\
\begin{subfigure}{\myfigwidth}
  \centering
  \includegraphics[width=\myplotwidth]{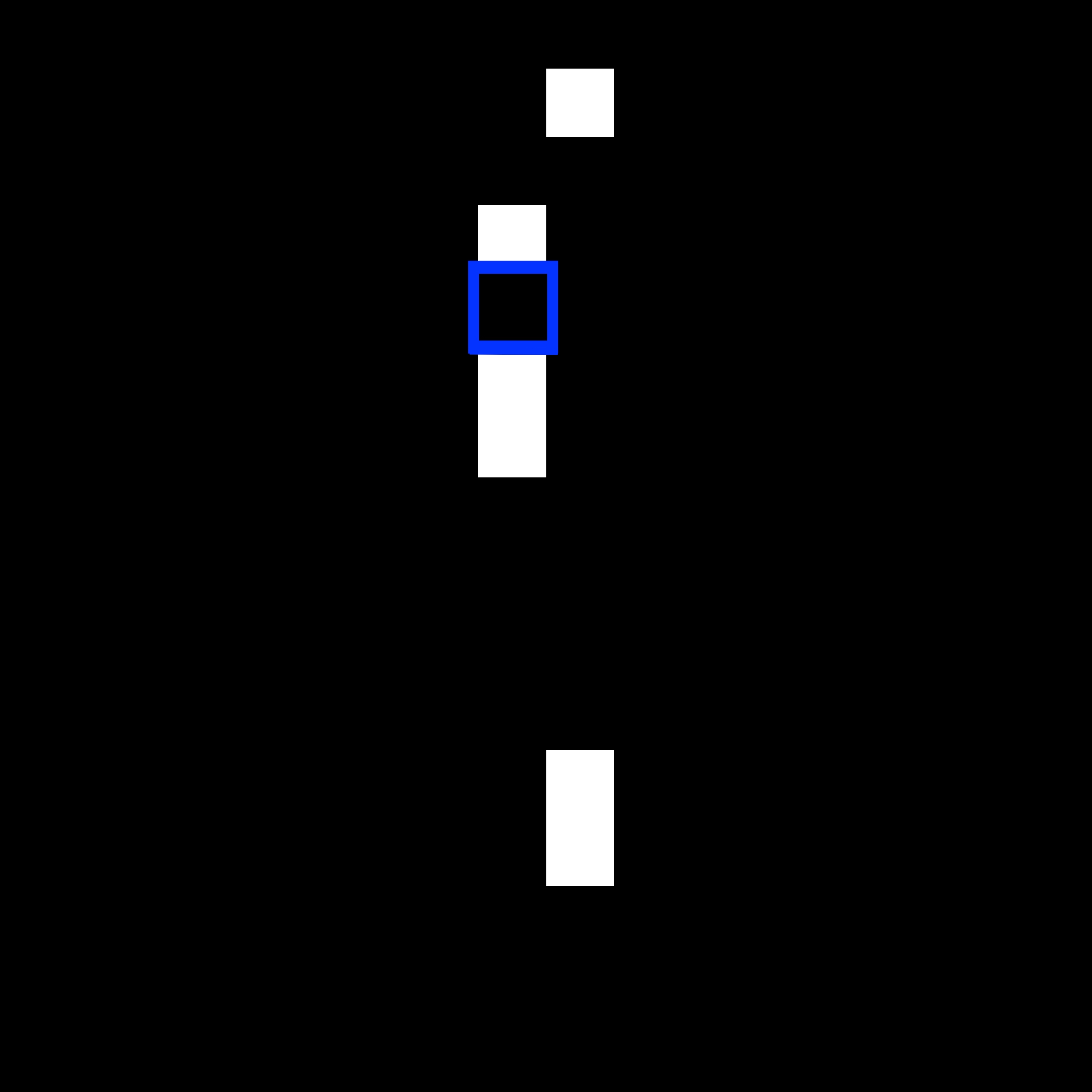}
  \caption{classified as digit-2 }\label{fig:sub5}
\end{subfigure}%
\begin{subfigure}{\myfigwidth}
  \centering
  \includegraphics[width=\myplotwidth]{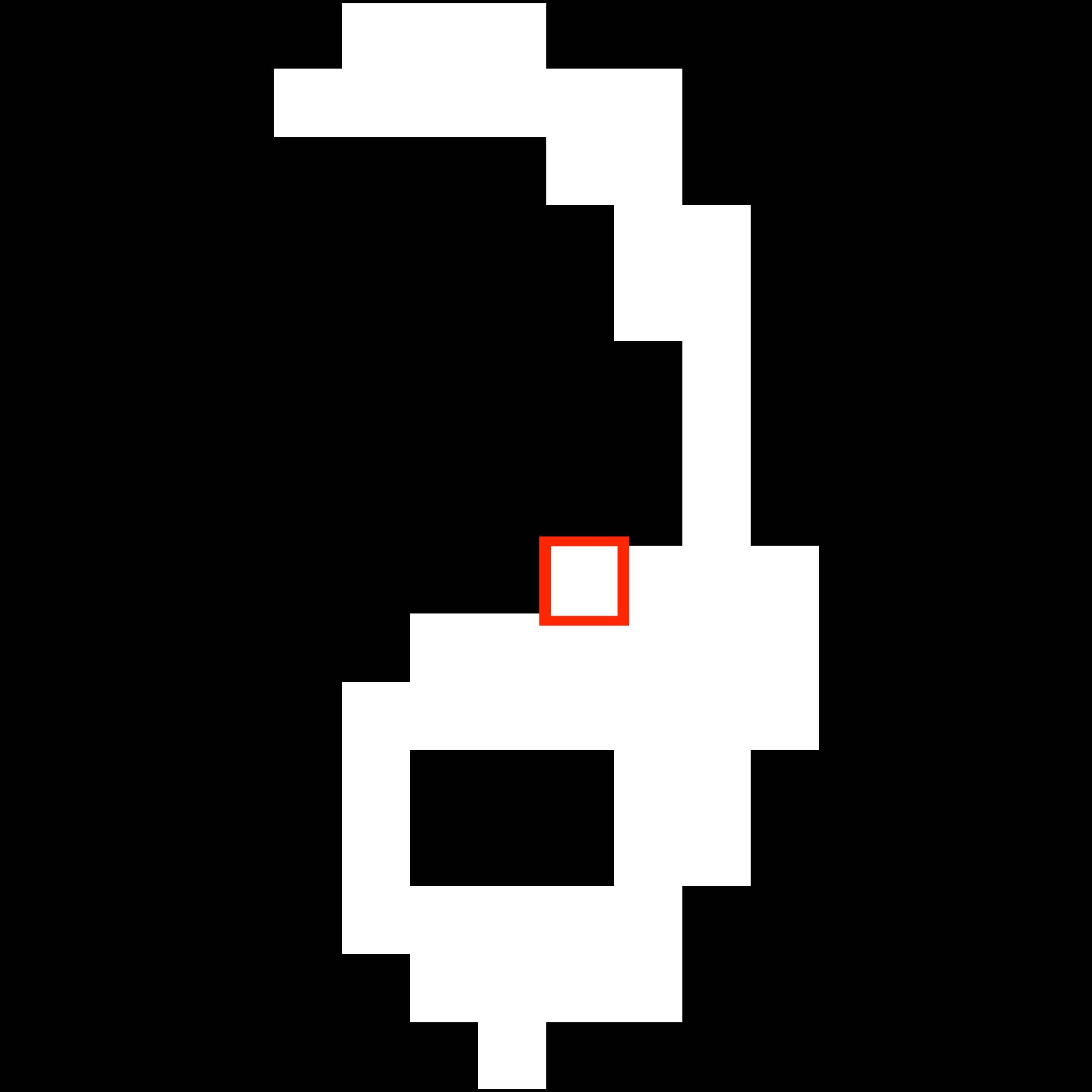}
  \caption{classified as digit-1}\label{fig:sub6}
 \end{subfigure}%
\caption{Visualizations of robustness.}
\label{fig:digit}
\end{figure}

Next, we consider in more depth Net 2, which again had a test set accuracy of \(96.93\%.\)
First, we visualize the most robust and the least robust instances of the CNN.  Figures~\ref{fig:sub1}~\&~\ref{fig:sub8} depict an example of a most robust digit-1 and digit-2, from the testing set.  Similarly Figures~\ref{fig:sub3}~\&~\ref{fig:sub4} depict an example of a least robust digit-1 and digit-2, both having robustness 1.  For these latter two instances, it suffices to flip a single pixel in each image, for the classifier to switch its label.  These perturbations are given in Figures~\ref{fig:sub5}~\&~\ref{fig:sub6}.  Finding training examples that have low-robustness can help finding problematic or anomalous instances in the dataset, or otherwise indicate weaknesses of the learned classifier.  Finding training examples that have high-robustness provides an insight into which instances that the classifier considers to be prototypical of the class.

\subsection{Pseudo-Polynomial Neuron Compilation}

\begin{figure}[t]
\centering
\includegraphics[width=.45\linewidth]{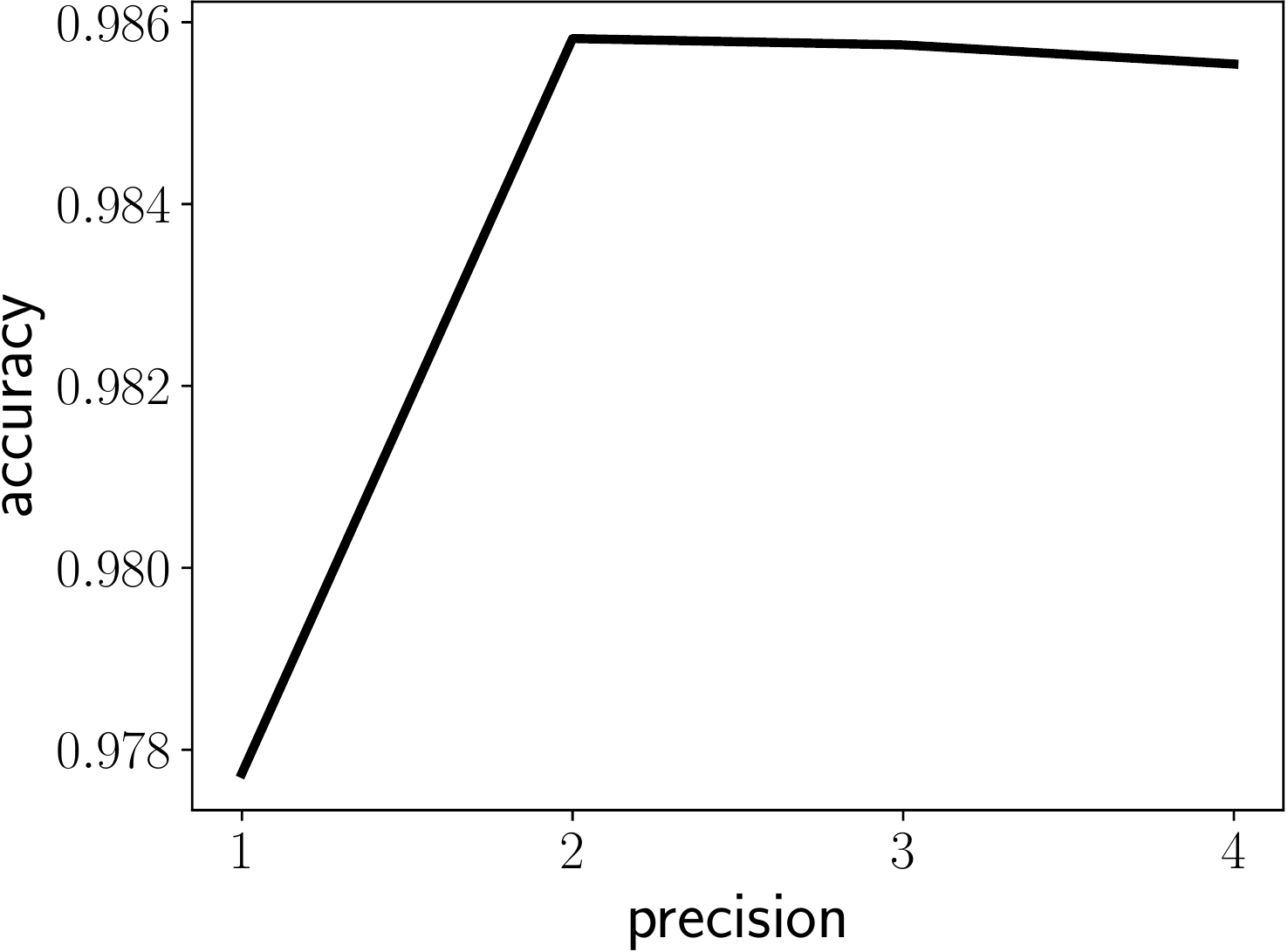}
\quad
\includegraphics[width=.45\linewidth]{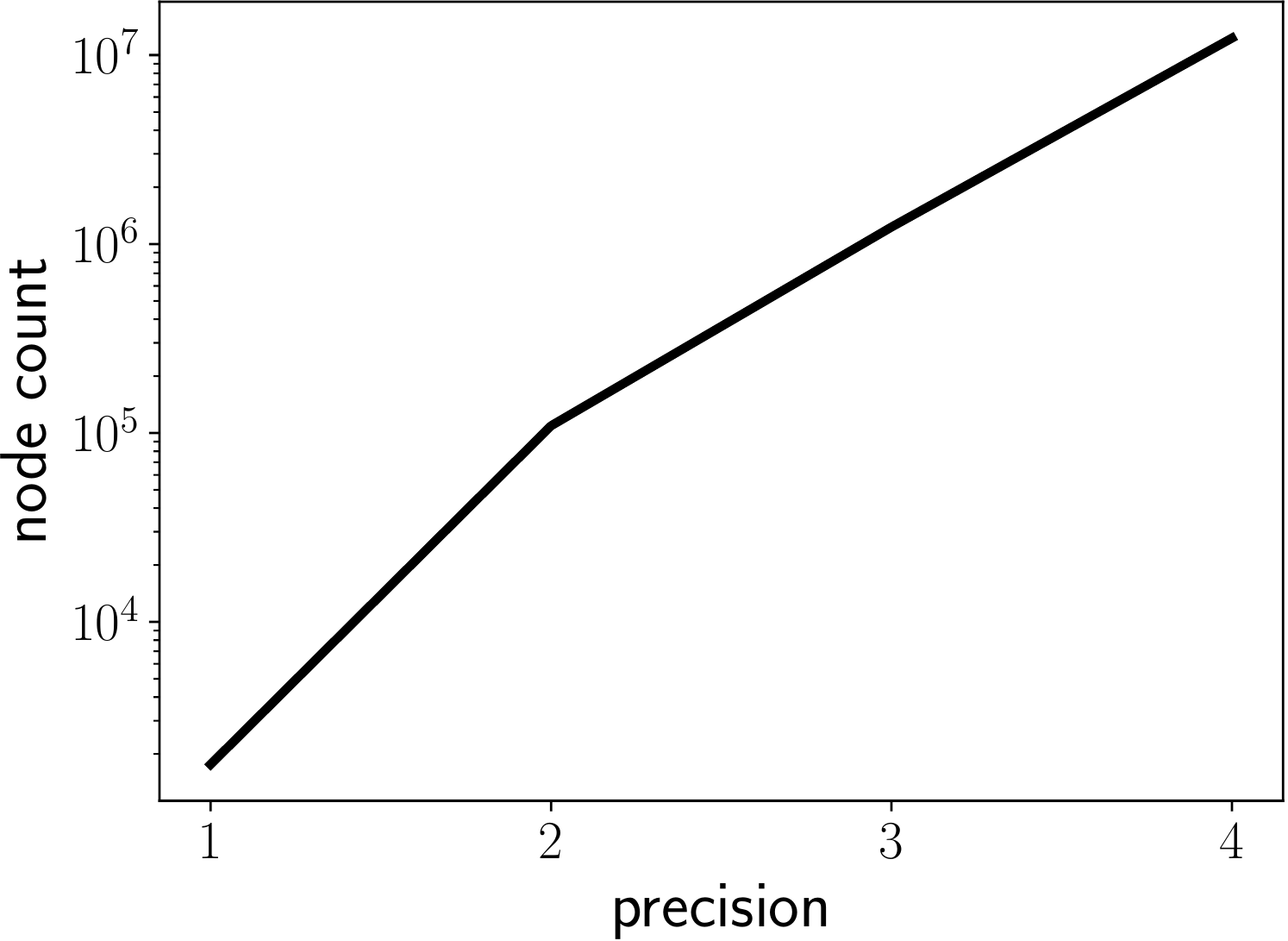}
\\
\includegraphics[width=.45\linewidth]{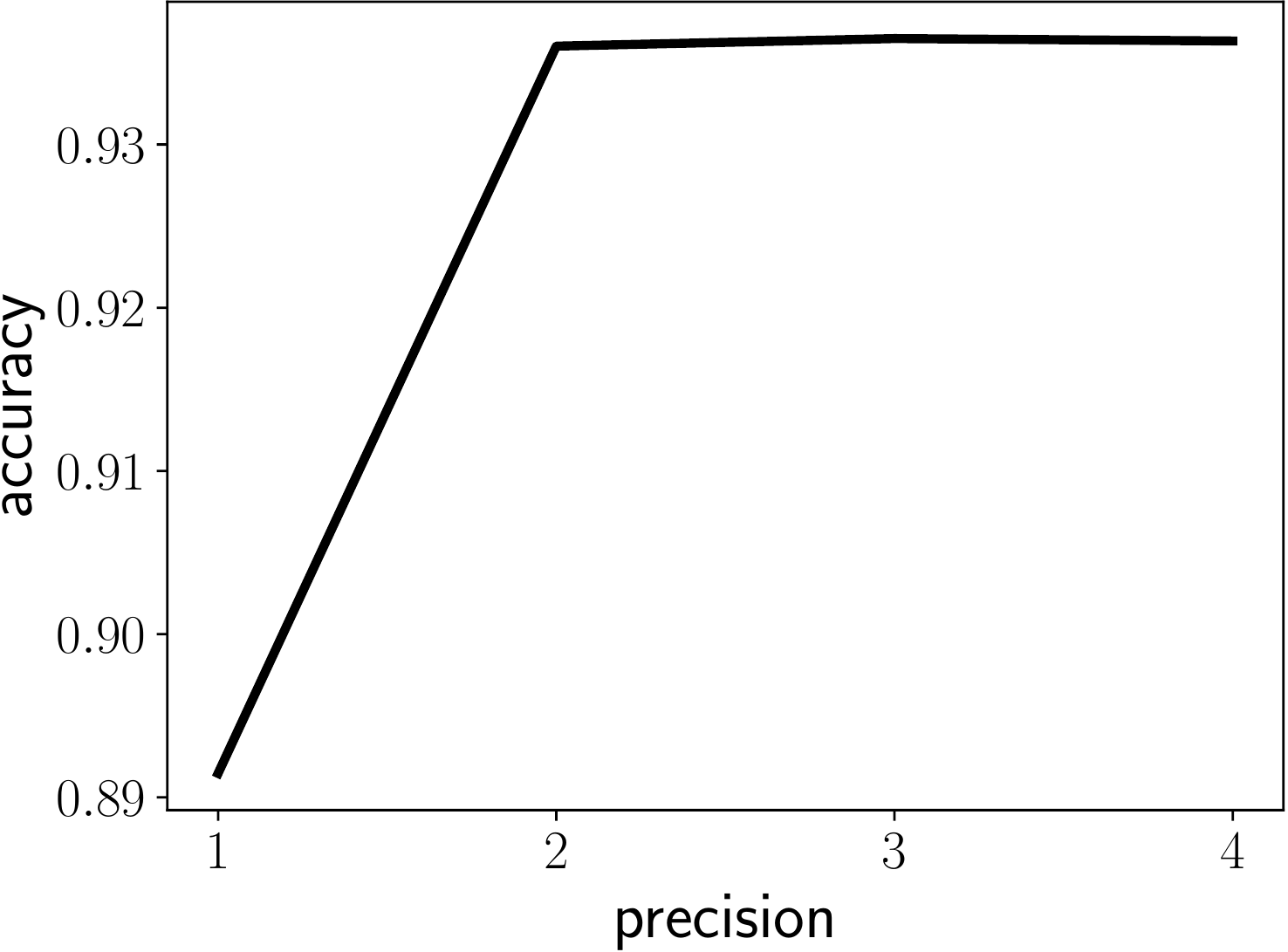}
\quad
\includegraphics[width=.45\linewidth]{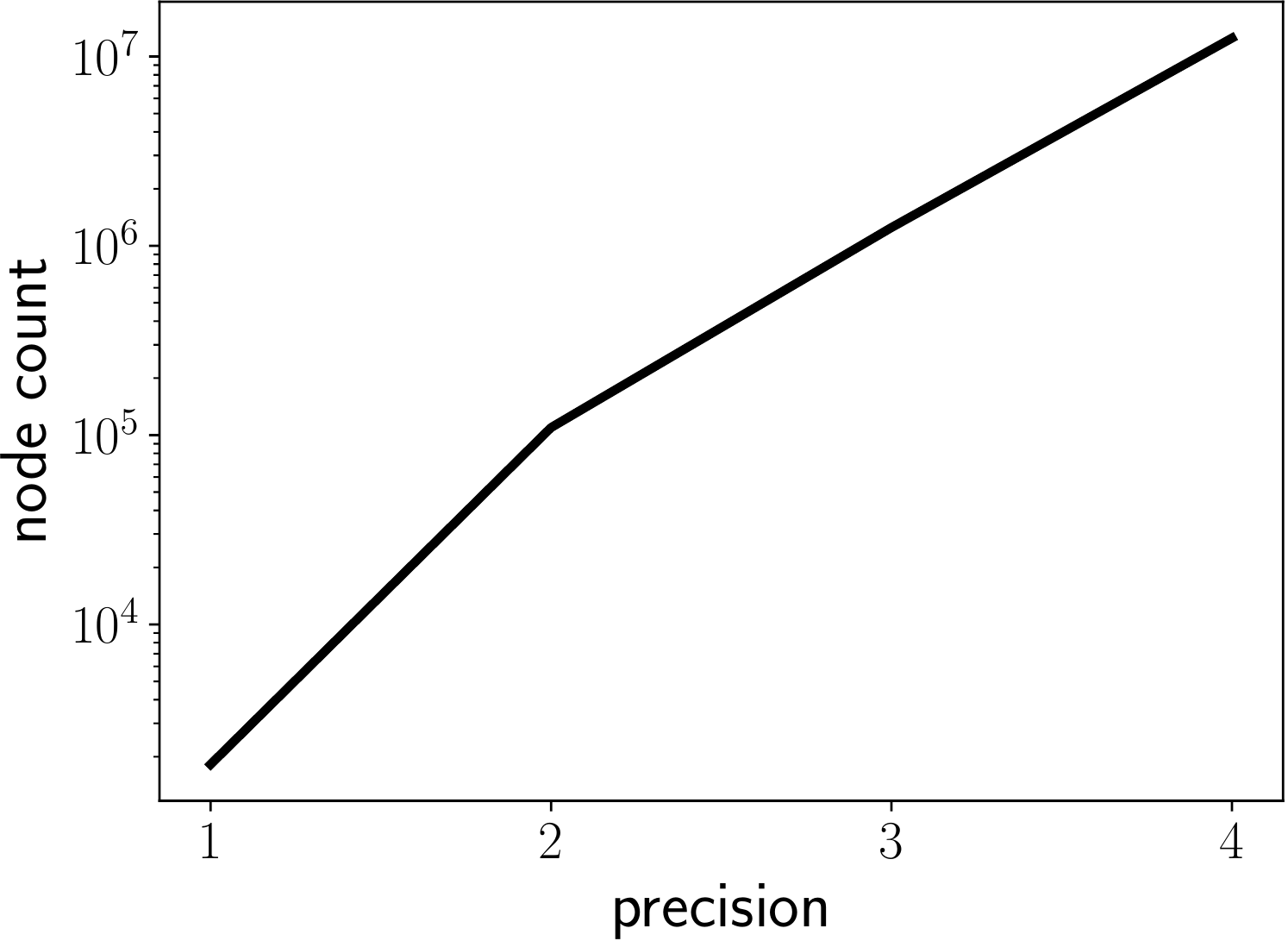}
\caption{Precision (in digits) versus test-set accuracy and node count.  Upper row is an average over 10 random parameter seeds for 0-vs-1.  Lower row is an overage over all 45 pairs of digits.
 \label{fig:exp}}
\end{figure}

Finally, we evaluate the pseudo-polynomial time algorithm of Theorem~\ref{theorem:pseudo-compile}, for compiling a neuron into an SDD (and more specifically into an OBDD).  This algorithm runs in polynomial time when the precision of the neuron's parameters are fixed.  We consider the same USPS digits dataset used in our case study.  Instead of training a neural network, we train a single neuron with \(16 \times 16 = 256\) inputs to classify a digit, which corresponds (roughly) to logistic regression.  Having \(n = 256\) inputs is well beyond the scope of the exact algorithm proposed by \cite{chanUAI03}, whose worst-case running time can be \(O(2^\frac{n}{2}).\)

Consider first Figure~\ref{fig:exp} (upper row).  Here, we trained a single neuron as a \(0\)-versus-\(1\) classifier, and averaged over 10 different parameter seeds.  On the \(x\)-axis, we increase the number of digits of precision in the weights, and on the \(y\)-axes we measure test-set accuracy and the size of the OBDD in terms of node count.  First, we find that with 2 digits of precision, we maintain a high \(98.6\%\) accuracy.  Next, we observe that even with a single digit of precision, we can still maintain around a \(97.8\%\) accuracy.  Next, we find that as we increase the digits of precision, the size of the resulting OBDD grows exponentially, as expected since compiling a neuron to an OBDD is NP-hard \cite{ShihCD18}.  All 10 cases failed (out-of-memory) for 5 digits of precision.  These observations suggest that a high degree of precision is not necessary to obtain good predictive performance, as observed by \cite{RastegariORF16,hubara2016binarized} as well, in the more extreme case of binarized weights.  We see a similar story in Figure~\ref{fig:exp} (lower row), where we have averaged over all 45 pairs of digits.

\section{Conclusion and Discussion} \label{sec:conclusion}

We proposed a knowledge compilation approach for explaining and verifying the behavior of a neural network.  We considered in particular neural networks with 0/1 inputs and step activation functions.  Such networks have neurons that correspond to Boolean functions.  The network itself also corresponds to a Boolean function, which maps an input feature vector into a class.  We showed how to compile the Boolean function of each neuron and the network itself into a tractable circuit in the form of a Sentential Decision Diagram (SDD).  We also introduced a pseudo-polynomial time algorithm that compiles neurons into tractable circuits, which can scale to neurons with hundreds of inputs when fixing the precision of neuron weights.
We also developed new queries and algorithms for analyzing the robustness of a Boolean function.  In a case study, we explained and analyzed the robustness of binary CNNs for classifying handwritten digits, and empirically evaluated the pseudo-polynomial time algorithm for compiling neurons.

As highlighted in Section~\ref{sec:kc}, the field of knowledge compilation studies the trade-offs between succinctness (the size of a circuit) and tractability (the number of polytime queries supported by the circuit).  In the context of explainable AI, this amounts to a trade-off between scalability and amenability to analysis. In terms of computational complexity, queries such as instance-based robustness correspond to the complexity class NP and can therefore
be tackled using SAT/SMT/MILP approaches.  Queries such as model-based robustness correspond to the complexity class PP and require approaches based on model counting and knowledge compilation (such as the approach we proposed in this paper). Since \(\mathrm{NP} \subseteq \mathrm{PP}\), approaches based on SAT solving (NP) are inherently more scalable, e.g., \cite{KatzBDJK17,NarodytskaKRSW18,IgnatievNM19a}, whereas approaches based on model counting and knowledge compilation offer more powerful types of analyses \cite{ShihCD18,BalutaSSMS19,Audemardetal20,DarwicheHirth20a}. Another class of promising approaches are the ones based on approximate model counting~\cite{BalutaSSMS19}, which represent an interesting compromise between scalability and analysis.

\appendix

\section{Proof of Theorem~\ref{theorem:pseudo-compile}} \label{sec:pseduo-compile}

\shrink{ 
Consider the following example of a neuron, from Section~\ref{sec:neurons}:
\[
1.15 \cdot A + 0.95 \cdot B - 1.05 \cdot C \ge 0.52 
\]
We can multiply both sides by \(100\) to obtain a neuron with an equivalent decision function, but with integer weights.
\[
115 \cdot A + 95 \cdot B - 105 \cdot C \ge 52 
\]
Let \(W = 115 + 95 + 105 + 52 = 367\) denote the sum of the absolute values of the neuron parameters, \(\), and let \(n = 3\) denote the number of variables.  Suppose we have variable ordering \(A,B,C\).  When we set \(A\) to 1 we obtain another neuron with one less input:
\[
95 \cdot B - 105 \cdot C \ge -63.
\]
When we set \(A\) to 0 we obtain the following neuron:
\[
95 \cdot B - 105 \cdot C \ge 52.
\]
These neurons differ only in the threshold being used.
}

Consider a neuron with inputs \(X_1,\ldots,X_n\).  Setting the inputs \(X_1,\ldots,X_i\) results in a smaller sub-classifier (or sub-neuron) over inputs \(X_{i+1},\ldots,X_n\).  No matter how we set the inputs \(X_1,\ldots,X_i\) the resulting sub-classifier is identical (has the same weights) except for the threshold being used.  Two different settings of inputs \(X_1,\ldots,X_i\) may lead to identical sub-classifiers with the same threshold.  Say we set variables from \(X_1\) to \(X_n\).  There are at most \(2W\) possible valid thresholds, so there are at most \(2 n W\) possible sub-classifiers that we can see while setting variables.

Consider an \(n \times W\) matrix \(A\) where cell \(A[i][j]\) is associated with the sub-classifier where variable \(X_i\) is about to be set, and where \(j\) is the threshold being used.  If \(X_i\) is set to \(1\), we obtain the sub-classifier at \(A[i+1][j+w_i]\) where \(w_i\) is the (integer) weight of feature \(X_i\).  If \(X_i\) is set to \(0\), we obtain the sub-classifier at \(A[i+1][j].\)  Each cell \(A[i][j]\) thus represents an OBDD node for the corresponding sub-problem, whose hi- and lo-children are known.  We add an \((n+1)\)-th layer where every sub-classifier with threshold below 0 is \(\bot,\) and where every sub-classifier at or above 0 is \(\top\).  The root of the original neuron's OBDD is then found at \(A[1][T]\), which we can extract and reduce if needed.

The size of matrix \(A\) bounds the size of the OBDD to \(O(nW)\) nodes.  Further, it takes constant time to populate each entry, and hence \(O(nW)\) time to construct the OBDD.

Finally, we note that the above construction also follows the proof of Theorem~\ref{theorem:compile} given by \cite{chanUAI03}, based also on identifying equivalence classes of sub-classifiers.  Thus, we can view Theorem~\ref{theorem:pseudo-compile} as a tightening of the bounds of Theorem~\ref{theorem:compile} for a special case (integer weights).

\section*{Acknowledgements}

This work has been partially supported by NSF grant \#ISS-1910317,
ONR grant \#N00014-18-1-2561, and DARPA XAI grant \#N66001-17-2-4032.

\bibliography{bib/aaai19}
\bibliographystyle{kr}

\end{document}